%% file: main.tex
\documentclass[twoside,11pt]{article}

\usepackage[abbrvbib, preprint]{jmlr2e}

\RequirePackage{amsmath,amsfonts,amssymb}
\usepackage{float}
\usepackage{latexsym}
\usepackage{graphicx}
\usepackage{mathrsfs}
\usepackage{enumerate}
\usepackage{mathabx}
\usepackage[dvipsnames]{xcolor}

\def\argmin{\operatorname{argmin} \displaylimits}

\def\mP{\mathcal{P}}

\def\mS{\mathcal{S}}
\def\mT{\mathcal{T}}
\def\mI{\mathcal{I}}

\def\T{\top}

\def\mS{\mathcal{S}}

\def\tr{\text{tr}}
\def\tbU{\widetilde{\bU}}
\def\tbV{\widetilde{\bV}}

\def\bF{\mathbold{F}}
\def\tvec{\text{vec}}

\newcommand\lfR[1]{{\color{red} #1}}
\newcommand\lfB[1]{{\color{blue} #1}}

\input{04-pre-am-short.tex}

\ShortHeadings{Projected Robust PCA}{Feng and Wang}
\firstpageno{1}

\begin{document}
	
	\title{Projected Robust PCA with Application to \\ Smooth Image Recovery}
	
	\author{\name Long Feng \email longfeng@cityu.edu.hk \\
		\name Junhui Wang \email j.h.wang@cityu.edu.hk \\
		\addr School of Data Science\\
		City University of Hong Kong\\
		Kowloon Tong, Hong Kong
	}
	
	\editor{}
	
	\maketitle
	
	\begin{abstract}
		Most high-dimensional matrix recovery problems are studied under the assumption that the target matrix has certain intrinsic structures. For image data related matrix recovery problems, approximate low-rankness and smoothness are the two most commonly imposed structures. For approximately low-rank matrix recovery, the robust principal component analysis (PCA) is well-studied and proved to be effective. For smooth matrix problem, 2d fused Lasso and other total variation based approaches have played a fundamental role. Although both low-rankness and smoothness are key assumptions for image data analysis, 
		the two lines of research, however, have very limited interaction. Motivated by taking advantage of both features, we in this paper develop a framework named projected robust PCA (PRPCA), under which the low-rank matrices are projected onto a space of smooth matrices.
		Consequently, a large class of image matrices can be 
		decomposed as a low-rank and smooth component plus a sparse component. A key advantage of this decomposition is that the dimension of the core low-rank component can be significantly reduced. Consequently, our framework is able to address a  problematic bottleneck of many low-rank matrix problems: singular value decomposition (SVD) on large matrices.  Theoretically, we provide explicit statistical recovery guarantees of PRPCA and include classical robust PCA as a special case.
	\end{abstract}
	
	
	\section{Introduction}
	
	In the past decade, high-dimensional matrix recovery problems have drawn numerous attentions in the communities of statistics, computer science and electrical engineering due to its wide applications, particularly in image and video data analysis.  
	Notable problems include face recognition \citep{parkhi2015deep}, motion detection in surveillance video \citep{candes2011robust}, brain structure study through fMRI \citep{maldjian2003automated}, etc.
	In general, most studies on high-dimensional matrix recovery problems are built upon the assumption that the target matrix has certain intrinsic structures. For image data related problems,  the two most commonly imposed structures are 1) approximate low-rankness and 2) smoothness. 
	
	The approximate low-rankness refers to the property that the target matrix can be  
	decomposed as a low-rank component plus a sparse component. Such matrices have been intensively studied since the seminal work of {\it robust principal component analysis} (RPCA, \citealt{candes2011robust}). The RPCA was originally studied under the noiseless setting and has been extended to the noise case by \citet{zhou2010stable}. Moreover, the Robust PCA has also been intensively studied for matrix completion problems with partially observed entries. For example, in \citet{wright2013compressive}, \citet{klopp2017robust} and \citet{chen2020bridging}.
	

	On the other hand, when a matrix is believed to be smooth, {\it Total Variation} (TV) based approach has played a fundamental role since the pioneering work of \citet{rudin1992nonlinear} and \citet{rudin1994total}.
	The TV has been proven to be effective in preserving image boundaries/edges.
	In statistics community, a well-studied TV approach is the {\it 2d fused Lasso} \citep{tibshirani2005sparsity}, which penalizes the total absolute difference of adjacent matrix entries using $\ell_1$ norm. The 2d fused Lasso has been shown to be efficient when the target matrix is piecewise smooth. More recently, a TV based approach was also used in image-on-scalar regression  to promote the piecewise smoothness of image coefficients \citep{wang2017generalized}.
	
	Although both low-rankness and smoothness are key assumptions for image data analysis, the two lines of study, however, have very limited interaction.  
	On the other hand, the matrices that are both approximately low-rank and smooth not only commonly exist in image data, it also exists in video analysis. 
	Consider a stacked video surveillance matrix---obtained by stacking each video frame into a matrix column. \citet{candes2011robust} demonstrates that this matrix is approximately low-rank: the low-rank component corresponds to the stationary background and the sparse component corresponds to the moving objects. However, a critical but often neglected fact is that the low-rank component is roughly column-wise smooth. In other words, each column of the low-rank matrix is roughly the same---because they all represent the same background. In this case, the original matrix is the superposition of a low-rank and smooth component and a  sparse component. 
	How to effectively take advantage of both assumptions? This motivates our study in this paper.
	
	\subsection{This paper}
	we propose the following model to build a bridge between the approximate low-rankness and smoothness in high-dimensional matrix recovery problem
	\bel{model2}
	\bZ &=& \bTheta +\bE,
	\cr \bTheta &=& \bP\bX_0\bQ^\T + \bY_0.
	\eel
	Here $\bZ\in\mathbb{R}^{M\times N}$ is the observed matrix with unknown mean matrix $\bTheta$ and noise $\bE$, $\bX_0\in\mathbb{R}^{m\times n}$ is an unknown  {\bf low-rank} matrix, $\bY_0\in\mathbb{R}^{M\times N}$ is an unknown {\bf sparse} matrix, $\bP\in\mathbb{R}^{M\times m}$ and $\bQ\in\mathbb{R}^{N\times n}$ are respectively certain ``row-smoother'' and ``column-smoother'' matrix that will be discussed in detail later. 
	The target is to recover the unknown matrices $\bX_0$, $\bY_0$ and the resulting $\bTheta$.
	We refer model (\ref{model2}) as the Projected Robust Principal Component Analysis (PRPCA) as the low-rank component in model (\ref{model2}) is projected onto a constrained domain. 
	
	We study the following convex optimization problem to estimate the pair $(\bX_0,\bY_0)$ and account for the low-rankness of $\bX_0$ and sparseness of $\bY_0$ in PRPCA, 
	\begin{align}\label{obj}
	(\hbX, \hbY)\in \argmin_{\substack{\bX\in \mathbb{R}^{n\times m}\\ \bY\in  \mathbb{R}^{N \times M}}} \frac{1}{2} \left\|\bZ-\bP\bX\bQ^\T-\bY \right\|_F^2 +\lambda_1 \|\bX\|_* +\lambda_2\|\bY\|_{\tvec(1)},
	\end{align}
	where $\lam_1$ and $\lam_0$ are the regularization parameters,  
	$\|\cdot\|_*$ is the nuclear norm (sum of eigenvalues) and $\|\cdot\|_{\tvec(1)}$ is the entrywise $\ell_1$-norm. 
	


	With different pairs of $(\bP, \bQ)$ and sparsity assumption on $\bY_0$,  model (\ref{model2}) includes many popular existing models. For example, when $\bP=\bI_N$ and $\bQ=\bI_M$ are identity matrices and $\bY_0$ is entrywise sparse, model (\ref{model2}) reduces to the classical RPCA  and the convex optimization problem (\ref{obj}) reduces to the noisy version of principal component pursuit (PCP, \citealt{candes2011robust}). When $\bP$ is a general matrix, $\bQ$ is the identity matrix and $\bY_0$ is columnwise sparse, model (\ref{model2}) reduces to the robust reduced rank regression studied by \citet{she2017robust}. Under such case, the $\|\cdot\|_{\tvec(1)}$ norm in (\ref{obj}) can be replaced by a mixed $\|\cdot\|_{2,1}$ to account for the columnwise sparsity of $\bY_0$ and our analysis below can be rephrased easily. Here for any matrix $\bM$,  $\|\bM\|_{2,1}=\sum_{j}\|M_{\cdot,j}\|_2^2$.

	As mentioned before, our study of PRPCA is motivated by taking advantage of both low-rankness and smoothness features of image data. In this paper, we show that the recovery accuracy of RPCA can be improved significantly when introducing the ``row-smoother'' and ``column-smoother'' matrices $\bP$ and $\bQ$. Beyond recovery accuracy,  the PRPCA also brings computational advantages compared to RPCA.	
	Indeed, the computation of RPCA or other low-rank matrix related problems usually involves iterations of singular value decomposition (SVD), which could be a problematic bottleneck for large matrices \citep{hastie2015matrix}. For smooth matrix recovery, the TV based approaches also posed great computational challenges.
	On the contrary, when we are able to combine the low-rankness with smoothness, problem (\ref{obj}) allows us to find a low-rank matrix 
	of dimension $n\times m$, rather than the original matrix with dimension $N\times M$.
	 As to be demonstrated in Section \ref{sec:smooth}, the ``smoother'' matrices we considered are mostly ``tall and thin'' matrices, i.e., $N\ge n$ and $M\ge m$. 
	That is to say, we are allowed to find a much smaller low-rank matrix and thus the computational cost are reduced.
	 A real image data analysis in Section \ref{sec:lenna} shows that the computation of PRPCA with $n=N/2$ and $m=N/2$ could be more than 10 times faster than RPCA while also achieves better recovery accuracy.

	More specifically, we in this paper study the theoretical properties of model (\ref{model2}) and the convex optimization problem (\ref{obj}) with general matrices $\bP$ and $\bQ$.
	Specifically, 
	we provide explicit theoretical error bounds for the estimation of the sparse component $\bY_0$ and low-rank component $\bP\bX_0\bQ^\T$ with general noise matrix $\bE$. Our results 
	includes \citet{hsu2011robust} as a special case, where the statistical properties of classical RPCA is studied.
	The key in our analysis of (\ref{obj}) is a careful construction of a dual certificate through a least-squares method. 
	In addition, a proximal gradient algorithm and its accelerated version are developed to implement (\ref{obj}). Furthermore, a comprehensive simulation study along with a real image data analysis further demonstrate the superior performance of PRPCA in terms of both recovery accuracy and computational benefits. 

	
	\subsection{Notations and Organizations}
	A variety of matrix and vector norms are used in this paper. 
	For a vector $\bv=(v_1,...v_p)^\T$, $\|\bv\|_q=\sum_{1\le j\le p} (\|v_j\|^q)^{1/q}$ is the $\ell_q$ norm, $\|\bv\|_0$ the $\ell_0$ norm (number of nonzero entries).  For a matrix $\bM=\{M_{i,j}, 1\le i\le n, 1\le j\le m\}$,   $\|\bM\|_{\tvec(q)}=\big(\sum_{i,j}|M_{i,j}|^q\big)^{1/q}$ is the entry-wise $\ell_q$-norm. In particular, $\|\bM\|_{\tvec(2)}$ is the Frobenius norm and also denoted as $\|\bM\|_F$, $\|\bM\|_{\tvec(0)}$ is the number of non-zero entries in $\bM$. Moreover, $\|\bM\|_q=\left[\sum_{i}\sigma^q_{i}(\bM)\right]^{1/q}$ is the Schatten $q$-norm, where $\sigma_i(\bM)$ are the singular values. In particular, $\|\bM\|_1$ is the nuclear norm (sum of the singular values) and also denoted as $\|\bM\|_*$. Furthermore, $\|\bM\|_{2,1}=\sum_{j}\|M_{\cdot,j}\|_2^2$ is a mixed $\ell_{2,1}$ norm. In addition, $\bM_{i,\cdot}$ is the $i$-th row of $\bM$,  $\bM_{\cdot,j}$ is the $j$-th column, $\tvec(\bM)$ is the vectorization of $\bM$, $\sigma_{\min}(\bM)$ and  $\sigma_{\max}(\bM)$ are the smallest and largest singular values, respectively, and $\bM^+$ is the Moore-Penrose inverse of $\bM$.
	Finally, we use $\bI_n$ to denote an identity matrix of dimension $n\times n$, and $\otimes$ to denote the Kronecker product. 
	
	The rest of the paper is organized as follows. Section \ref{sec:smooth} introduces the interpolation matrices based PRPCA. In Section \ref{sec:computation} we discuss the computation of (\ref{obj}) with proximal gradient algorithm. Section \ref{sec:theorem} provides main theoretical results, a sharp finite sample statistical recovery guarantee is provided for PRPCA. We conduct a comprehensive simulation study in Section \ref{sec:sim} and a real image data analysis in Section \ref{sec:lenna}. Section \ref{sec:conclusion} includes conclusions and future directions.
	\section{Projected robust PCA with smoothing matrices}
	In this section, we consider two types of smoothing mechanisms on the low-rank matrix: data-independent smoothing and data-dependent smoothing. Moreover, we provide general assumptions of the ``smoother matrices'' $\bP$ and $\bQ$ that our analysis can be applied for. 
	\subsection{The data-independent smoothing and interpolation matrix}\label{sec:smooth}

	\begin{definition}\label{def-2}
		Let $N$ be an even integer and $n=N/2$  \footnote{When $N$ is odd, we can let  $n=(N+1)/2$ and a slightly different interpolation matrix can be defined in a similar way.} . We define the normalized interpolation 
		matrix $\bJ_N$ of dimension $N\times n$ as
		\bel{def-2-1}
		\bJ_N= \frac{1}{2}\begin{pmatrix}
			2& 0 &0&0&\cdots&0& 0\\
			2&0  &0 &0&\cdots &0 &0\\ 
			1&1 &0 &0&\cdots &0 &0\\
			0&2 &0 &0&\cdots &0 &0\\
			0&1 &1 &0&\cdots &0 &0\\
			&&&&\cdots \\
			0&0&0&0&\cdots& 0&2 \\ 
		\end{pmatrix}\in \mathbb{R}^{N \times  n},
		\eel
		i.e., the $j$-th column of $\bJ_N$ is
		\bes
		\{\bJ_N\}_{.,j}=(\underbrace{0,\ldots,0}_{2(j-1)},\frac{1}{2},1,\frac{1}{2},\underbrace{0,\ldots,0}_{N-2j-1})^\T,\ \  j=2,...,n-1,
		\ees
		$\{\bJ_N\}_{.,1}=(1,1,\frac{1}{2},0,...,0)$ and $\{\bJ_N\}_{.,n}=(0,...,0,\frac{1}{2},1)$.
	\end{definition}

	The interpolation matrices $\bJ_N$ and $\bJ_M$ play the role of ``row smoother'' and ``column smoother'', respectively. That is to say, when $\bP=\bJ_N$,  $\bP\bU\in\mathbb{R}^{N\times M}$ is a row-wisely smooth matrix for any matrix $\bU\in \mathbb{R}^{n\times M}$, i.e., except the first row (boundary effect), any odd row of $\bP\bU$ is the average of adjacent two rows,
	\bel{smooth1}
	(\bP\bU)_{i, \cdot } =\frac{1}{2}\left\{(\bP\bU)_{i-1, \cdot} + (\bP\bU)_{i+1, \cdot}\right\}, 
	\eel
	for row $i=3,5,\ldots, N-1$, and for the boundary row,  
	\bel{smooth2}
	(\bP\bU)_{1, \cdot } =(\bP\bU)_{2, \cdot}.
	\eel
	Also, when $\bQ=\bJ_M$, $\bV\bJ_M^\T\in\mathbb{R}^{N\times M}$ is a column-wisely smooth matrix for any matrix $\bV\in\mathbb{R}^{N\times m}$:
	\bes
	(\bV\bQ^\T)_{\cdot, j} = \frac{1}{2}\left\{(\bV\bQ^\T)_{\cdot, j-1} +(\bV\bQ^\T)_{\cdot, j+1}\right\}, 
	\ees
	for column $j=3,5,\ldots, M-1$, and for the boundary column,
	\bes 
	(\bV\bQ^\T)_{\cdot, 1} = (\bV\bQ^\T)_{\cdot, 2}.
	\ees
	As a consequence, $\bP\bX_0\bQ^\T$ in model (\ref{model2}) is a smooth matrix both row-wisely and column-wisely.

We note that the interpolation matrix has also been used in other image analysis literature. For example, when implementing RPCA, \citet{hovhannisyan2019fast} used interpolation matrix to build connections between the original ``fine'' model and a smaller ``coarse'' model and reduce the computational burden of RPCA. Their work is mainly from computational perspective, but the principle behind is the same: by applying SVD in models of lower-dimension, the computational burden can be significantly reduced.

Indeed, by introducing the smoothing matrices $\bP$ and $\bQ$, the PRPCA enjoys significant computational advantages compared to the standard RPCA. 
When $\bP$ and $\bQ$ are interpolation matrices, we are allowed to find a low-rank matrix 
of dimension $N/2\times M/2$, rather than the original matrix with much higher-dimension $N\times M$. 
Considering that computing a low-rank matrix usually involves SVD, the computational advantage that (\ref{obj}) brings is even more significant. 
More aggressively, we may further interpolate the low-rank matrix by using a double interpolation matrix
\bes
(\bP,\bQ)=(\bJ_N\times \bJ_{N/2}, \ \bJ_M\times \bJ_{M/2}).
\ees 
This would allow us to find a low-rank matrix of even lower dimension, $N/4\times M/4$, and further reduce the computational burden.
	
		The smoothing mechanism here is data-independent in the sense that the odd row (column) of the low-rank matrix is an equal-weights average of the adjacent two rows (columns). Such mechanism can be modified to a data-dependent smoothing approach, which will be introduced in the next subsection. On the other hand, we observe that the data-independent averaging performs consistently well across a large range of image recovery problem in our simulation and real data analysis. We refer to Section \ref{sec:sim} and \ref{sec:lenna} for more details.
		
	\subsection{The data-dependent smoothing with robust linear regression}\label{sec:lrsmooth}
 In this subsection, we introduce a data dependent smoothing mechanism based on robust linear models. 
We start with introducing the data-dependent interpolation matrix.

	\begin{definition}\label{def-3}
	Let $N$ be an even integer and $n=N/2$. Let $w_{ur}^i$, $w_{lr}^i$, $i=1,\ldots,n-1$ represent the $i$-th unknown upper-row weight and lower-row weight, respectively. We define the generalized interpolation  
	matrix $\bK_N$ of dimension $N\times n$ as
	\bel{def-3-1}
	\bK=\bK_N(\bW)= \begin{pmatrix}
		1& 0 &0&0&\cdots&0& 0\\
		1&0  &0 &0&\cdots &0 &0\\ 
		w_{ur}^1&w_{lr}^1 &0 &0&\cdots &0 &0\\
		0&1 &0 &0&\cdots &0 &0\\
		0&w_{ur}^2 &w_{lr}^2 &0&\cdots &0 &0\\
		&&&&\cdots \\
		0&0&0&0&\cdots& 0&1 \\ 
	\end{pmatrix}\in \mathbb{R}^{N \times  n},
	\eel
	i.e., the $j$-th column of $\bK_N$ is
	\bes
	\{\bK_N\}_{.,j}=(\underbrace{0,\ldots,0}_{2(j-1)},w_l^{j-1},\ 1, \ w_u^j,\ \underbrace{0,\ldots,0}_{N-2j-1})^\T,\ \  j=2,...,n-1,
	\ees
	$\{\bK_N\}_{.,1}=(1,1,w_u^1,0,...,0)$ and $\{\bK_N\}_{.,n}=(0,...,0,w_{lr}^{n-1},1)$.
\end{definition}

The weights $w_{ur}^i$, $w_{lr}^i$ are unknown parameters and need to be estimated. Clearly, the normalized interpolation matrices $\bJ_N$ defined in the previous section is a special case of $\bK_N$ with the weights $w_{ur}^i=0.5$ and $w_{lr}^i=0.5$. As in (\ref{smooth1}) and (\ref{smooth2}), when $\bP=\bK_N$, $\bP\bU$ is a weighted smoothing matrix for any matrix $\bU$ of dimension $n\times M$:
	\bes
	(\bP\bU)_{i', \cdot } =w_{ur}^{i}(\bP\bU)_{i'-1, \cdot} + w_{lr}^{i}(\bP\bU)_{i'+1, \cdot}, 
	\ees
	and for the boundary row,  
	\bes
	(\bP\bU)_{1, \cdot } =(\bP\bU)_{2, \cdot}.
	\ees
	Similarly, we may define $\bK_M$ in the same way as $\bK_N$, but with $w_{ur}^i$, $w_{lr}^i$, $i=1,\ldots,n-1$ replaced by $w_{lc}^j$, $w_{rc}^j$, $j=1,\ldots,m-1$, representing left-column and right-column weights. Then, when $\bQ=\bK_M$, $\bV\bQ^\T\in\mathbb{R}^{N\times M}$ is a column-wisely smooth matrix for any $\bV\in\mathbb{R}^{N\times m}$.

Now we present the estimation procedure of $w_{ur}^i$ and $w_{lr}^i$, or the row smoothing matrix $\bK_N$. We only present the estimation of $\bK_N$ as the column smoothing matrix $\bK_M$, or $w_{lc}^j$, $w_{rc}^j$,  can be estimated in the same way and the details are omitted. Let $\bDelta=\bP\bX_0\bQ^\T$ denote the unknown low-rank and smooth matrix. 
Intuitively, if we treat the odd rows in $\bDelta$ as missing, then the values of $w_{ur}^i$ and $w_{lr}^i$ can be viewed as the linear weights when inserting the odd rows based on their neighborhood. Thus, it is ideal to estimate $w_{ur}^i$ and $w_{lr}^i$ based on $\bDelta$, in particular, based on the $(i'-1)$-th, $i'$-th and $(i'+1)$-th row of the  matrix $\bDelta$, or $\bDelta_{(i'-1):(i'+1),}$, with $i'=2i+1$. However, as $\bDelta$ is unobserved, $w_{ur}^i$ and $w_{lr}^i$ are unable to be estimated directly from $\bDelta$. While on the other hand, we note that $\bDelta=\bZ-\bY-\bE$. Considering that 1) $\bY$ is a sparse matrix and can be viewed as outliers from $\bDelta$, 2) $\bE$ is a noise matrix with small entry-wise magnitude, we propose to estimate $w_{ur}^i$ and $w_{lr}^i$ based on the observable $\bZ$ through robust linear regression to account for the outliers $\bY$. Specifically, we consider the following minimization problem:
\bes
(\hat{w}_{ur}^i,\hat{w}_{lr}^i)=\argmin_{w_{ur}^i, w_{lr}^i} \sum_{j=1}^M\rho(w_{ur}^i Z_{i'-1,j}+w_{lr}^i Z_{i'+1, j}-Z_{i',j}), \  \ i'=2i+1,\  \ i=1,\ldots, n-1.
\ees
Here $\rho(\cdot)$ is certain robust loss function. For example, we may take $\rho(\cdot)$ as the Huber loss, where
\bes
\rho(x)=\begin{cases}
	x^2,  &\text{if }  |x| \le \kappa,
	\cr 2\kappa |x|-\kappa^2,   &\text{if } |x| > \kappa.
\end{cases}
\ees	
The tuning parameter $\kappa>0$ is assumed to be given in our estimation of $w_{ur}^i$ and $w_{lr}^i$.  When $\bP=\bK_N$ and $\bQ=\bK_M$, we have the following property:

\begin{proposition}\label{prop-1}
	Let $\bP=\bK_N\in\mathbb{R}^{N\times n}$ and $\bQ=\bK_M\in\mathbb{R}^{M\times m}$ be the matrices in Definition \ref{def-3} with $N\ge 4$, $M\ge 4$ and any weights $w_{ur}^i$, $w_{lr}^i$, $w_{lc}^j$ and $w_{rc}^j$. Then for any nonzero matrix $\bX_0\in \mathbb{R}^{n\times m}$, 
	\begin{align}
	\|\bX_0\|_*<  \|\bP\bX_0\bQ^\T\|_* .\label{th-2-2}
	\end{align}
\end{proposition}

	When $\bTheta$ can be decomposed as in model (\ref{model2}), the RPCA is the optimization problem (\ref{obj}) with the nuclear penalty on $\bX$ replaced by that on $\bP\bX\bQ^\T$. Proposition \ref{prop-1} suggests that smaller penalty is applied in (\ref{obj}) compared to that of RPCA with the same $\lam_1$.

	\subsection{PRPCA with general $\bP$ and $\bQ$} 
	
	Although our study of PRPCA is motivated by smooth matrix analysis and resulting interpolation matrices, our results below work for general matrices $\bP$ and $\bQ$. 
	
	Specifically, we proceed our analysis to consider $\bP$  and $\bQ$ of full-column rank. 
Indeed, when the target is to recover $\bP\bX_0\bQ^\T$ as a whole instead of $\bX_0$, it is sufficient to consider $\bP$ and $\bQ$ of full column rank. This can be seen from the following arguments.
	For any $\bP\in\mathbb{R}^{N\times n}$ and $\bQ\in\mathbb{R}^{M\times m}$, there exists $r_1\le  \min(N, n)$ and $r_2\le  \min(M,m)$ and full column rank matrices $\bP_0\in\mathbb{R}^{N\times r_1}$, $\bQ_0\in\mathbb{R}^{M\times r_2}$ such that
	\bes
	\bP=\bP_0\bLambda, \ \ \ \ \bQ=\bQ_0\bOmega
	\ees
	holds for some $\bLambda\in\mathbb{R}^{r_1\times n}$ and $\bOmega\in\mathbb{R}^{r_2\times m}$. As a result, an alternative representation of model (\ref{model2}) with full column rank matrices $\bP_0$ and $\bQ_0$
	\bes
	\bZ= \bP_0(\bLambda\bX_0\bOmega^\T)\bQ_0^\T + \bY_0 +\bE,
	\ees
	Here the columns of $\bP_0$ (or $\bQ_0$) can be viewed as the ``factors'' of $\bP$ (or $\bQ$). This confirms the sufficiency of considering PRPCA with full-column rank matrices $\bP$ and $\bQ$. In the following sections, we derive properties of PRPCA with general $\bP$ and $\bQ$ of full-column rank. 
	

	\section{Computation with proximal gradient algorithm}\label{sec:computation}
	

	
	Given $\bP$ and $\bQ$, the problem (\ref{obj}) is a convex optimization problem. In this section, we show that it can be solved easily through a proximal gradient algorithm. 
	
	We first denote the loss and penalty function in problem (\ref{loss}) as
	\bel{loss}
	\mathcal{L}(\bX,\bY) = \frac{1}{2} \left\|\bZ-\bP\bX\bQ^\T-\bY \right\|_F^2,
	\eel
	and
	\bel{penalty}
	\mathcal{P}(\bX,\bY) = \lambda_1 \|\bX\|_* +\lambda_2\|\bY\|_{\tvec(1)},
	\eel
	respectively. Also, note that if we let $\bA=\bQ\otimes\bP$, the loss function (\ref{loss}) could be written as
	\bel{loss2}
	\mathcal{L}(\bX,\bY)  = \frac{1}{2}\left\| \bA \tvec(\bX)+\tvec(\bY)-\tvec(\bZ)\right\|_2^2.
	\eel
	
	To minimize $\mathcal{L}(\bX,\bY) + \mathcal{P}(\bX,\bY)$, we utilize a variant of Nesterov’s proximal-gradient method \citep{nesterov2013gradient}, which
	iteratively updates
	\bel{psi}
	(\hbX_{k+1}, \hbY_{k+1})\leftarrow \argmin_{\bX,\bY}\psi(\bX,\bY|\hbX_{k},\hbY_{k}),
	\eel
	where 
	\begin{align*}
		&\psi(\bX,\bY|\hbX_{k},\hbY_{k})\cr =&\mathcal{L}(\hbX_{k},\hbY_{k}) + \langle \nabla_{\bX} \mathcal{L}(\hbX_{k},\hbY_{k}) ,  \bX-\hbX_{k} \rangle + \langle \nabla_{\bY} \mathcal{L}(\hbX_{k},\hbY_{k}),  \bY-\hbY_k \rangle \cr &+\frac{L_k}{2} \left(\|\bX-\hbX_k\|_F^2+\|\bY-\hbY_k\|_F^2\right)+\mathcal{P}(\bX,\bY),
	\end{align*}
	$L_{k}$ is the step size parameter at step $k$, $\nabla_{\bX} \mathcal{L}(\hbX_k,\hbY_k)$ and $\nabla_{\bY} \mathcal{L}(\hbX_k,\hbY_k)$ are the gradients
	\bes
	\nabla_{\bX} \mathcal{L}(\hbX_k,\hbY_k)&=&\bP^\T(\bP\hbX_k\bQ^\T+\hbY_k-\bZ)\bQ,  \cr \nabla_{\bY} \mathcal{L}(\hbX_k,\hbY_k)&=&\bP\hbX_k\bQ^\T+\hbY_k-\bZ.
	\ees
	
	The proximal function $\psi(\bX,\bY|\hbX_k,\hbY_k)$ is much easier to optimize compared to $\mathcal{L}(\bX,\bY) + \mathcal{P}(\bX,\bY)$. In fact, a closed-form expression is available for the updates.
	\bes
	\hbX_{k+1} &=& \mathcal{SVT}\left(\hbX_{k}-\frac{1}{L_k}\nabla_{\bX} \mathcal{L}(\hbX_k,\hbY_k); \ \frac{\lam_1}{L_k}\right)
	\cr 
	\hbY_{k+1} &=& \mathcal{ST}\left(\hbY_{k}-\frac{1}{L_k}\nabla_{\bY} \mathcal{L}(\hbX_k,\hbY_k); \ \frac{\lam_2}{L_k}\right)
	\ees
	where $\mathcal{SVT}$ and $\mathcal{ST}$ are the Singular Value Thresholding and Soft Thresholding operators with specifications below.

	Given any non-negative number $\tau_1\ge 0$ and any matrix $\bM_1\in \mathbb{R}^{n\times m}$ with singular value decomposition $\bM_1=\bU\bSigma\bV^\T$, where $\bSigma=\diag(\{\sigma_i\}_{1\le i\le r})$, $\sigma_i\ge 0$, the SVT operater $\mathcal{SVT}(\cdot;\cdot)$, which was first introduced by \citet{cai2010singular},  is defined as
	\bes
	\mathcal{SVT} (\bM_1, \tau_1) &=&\argmin_{\bX\in\mathbb{R}^{n\times m}}\frac{1}{2}\|\bX-\bM_1 \|_F^2 +\tau_1 \|\bX\|_*\cr &=&\bU\mathcal{D}_{\tau_1} (\bSigma)\bV^\T,
	\ees
	where $\mathcal{D}_{\tau_1}(\bSigma)=\diag(\{\sigma_i-\tau_1\}_+)$.
	For any $\tau_2\ge 0$ and any matrix $\bM_2\in \mathbb{R}^{N\times M}$, the ST operator $\mathcal{ST}(\cdot;\cdot)$is defined as
	\bes
	\mathcal{ST} (\bM_2; \tau_2) &=& \argmin_{\bY\in\mathbb{R}^{N\times M}} \frac{1}{2} \|\bY-\bM_2\|_F^2+\tau_2\|\bY\|_{\tvec(1)} \cr &=&\sgn(\bM_2) \circ \left(|\bM_2|-\tau_2 {\bf 1_N}{\bf 1_M^\T}\right)_+. 
	\ees

	We summarize the proximal gradient algorithm for PRPCA in Table \ref{pg1}.
	\begin{table}[H]\label{pg1}
		\begin{center}
			\begin{tabular}{ll}\\ \hline
				\multicolumn{2}{l}{{\bf Algorithm 1:} Proximal gradient for PRPCA}   \\ \hline 
				\smallskip
				Given: &$\bZ\in\mathbb{R}^{N\times M}$,
				$\bP\in\mathbb{R}^{N\times n}$, $\bQ\in\mathbb{R}^{M\times m}$, $\lambda_1$ and $\lam_2$ \smallskip		\\
				Initialization: & $\hbX_0 = \hbX_{-1}=\bf{0}_{n\times m}$,
				$\hbY_0 = \hbY_{-1}=\bf{0}_{N\times M}$
				\smallskip		\\ 
				
				Iteration: 	& $\bG_k^Y = \bP\hbX_k\bQ^\T+\hbY_k-\bZ$ \\ 
				& $\bG_k^X = \bP^\T\bG_k^Y\bQ$ \\
				& $\hbX_{k+1}=\mathcal{SVT}\left(\hbX_k-(1/L^k) \bG^X_{k};\  (1/L_k)\lam_1\right)$, \\
				& $\hbY_{k+1}=\mathcal{ST}\left(\hbY_k-(1/L^k) \bG^Y_{k};\  (1/L_k)\lam_2\right)$, \\
				\hline
				\multicolumn{2}{l}{Note: $L_k$ can be taken as the reciprocal of a Lipschitz  constant for $\nabla \mathcal{L}(\bX,\bY)$ or }\\
				\multicolumn{2}{l}{determined by backtracking.}
			\end{tabular}
		\end{center}
	\end{table}
	The proximal gradient algorithm for PRPCA iteratively implements SVT and ST. Note that in the SVT step, the singular value decomposition is implemented on $\hbX_{k}-(1/L_k)\nabla_{\bX} \mathcal{L}(\hbX_k,\hbY_k)$, which is of dimension $n\times m$. Compared to the RPCA problem which requires singular value decomposition on matrices of much larger dimension $N\times M$, the PRPCA greatly reduces the computational cost. 
	
	Moreover, the proximal gradient can be further accelerated in a FISTA \citep{beck2009fast} style as in Algorithm 2 below. For all the simulation studies and real image data analysis is Section \ref{sec:sim} and \ref{sec:lenna}, we adopt the accelerated proximal gradient algorithm. 
	\begin{table}[H]
		\begin{center}
			\begin{tabular}{ll}\\ \hline
				\multicolumn{2}{l}{{\bf Algorithm 2:} Accelerated proximal gradient for PRPCA} \\ \hline 
				\smallskip
				Given: &$\bZ\in\mathbb{R}^{N\times M}$,
				$\bP\in\mathbb{R}^{N\times n}$, $\bQ\in\mathbb{R}^{M\times m}$, $\lambda_1$ and $\lam_2$ \smallskip		\\
				Initialization: & $\hbX_0 = \hbX_{-1}=\bf{0}_{n\times m}$,
				$\hbY_0 = \hbY_{-1}=\bf{0}_{N\times M}$, $t_0=t_1=1$ 					\smallskip		\\ 
				
				Iteration: &$\bF^X_{k} = \hbX_k+t_k^{-1}(t_{k-1}-1)(\hbX_k - \hbX_{k-1})$\\ 
				&$\bF^Y_{k} = \hbY_k+t_k^{-1}(t_{k-1}-1)(\hbY_k - \hbY_{k-1})$ \\
				& $\bG_k^Y = \bP\hbX_k\bQ^\T+\hbY_k-\bZ$ \\ 
				& $\bG_k^X = \bP^\T\bG_k^Y\bQ$ \\ 
				& $\hbX_{k+1}=\mathcal{SVT}\left(\bF^X_k-(1/L^k) \bG^X_{k};\  (1/L_k)\lam_1\right)$, \\
				& $\widehat{\bY}_{k+1}=\mathcal{ST}\left(\bF^Y_k-(1/L^k) \bG^Y_{k};\  (1/L_k)\lam_2\right)$, \\
				& $t_{k+1} = \{1+(1+4t_k^2)^{1/2}\}/2$ \\
				\hline
			\end{tabular}
		\end{center}
	\end{table}
	
	\section{Main theoretical results}\label{sec:theorem}
	In this section, we present our main theoretical results for recovering the PRPCA. Specifically, we provide sharp theoretical error bounds for the estimation of the low-rank and smooth component $\bP\bX_0\bQ^\T$ and the sparse component $\bY_0$ when $\bP$ and $\bQ$ are correctly specified. 
	Note that 
	\citet{hsu2011robust} studied the theoretical properties of RPCA, i.e., PRPCA with
	$\bP=\bI_N$ and $\bQ=\bI_M$ being identity matrices. Our results can be viewed as a generalization of theirs. 
	
	\subsection{Technique preparations}
For a target decomposition of $\bTheta=\bP\bX_0\bQ^\T + \bY_0$, we consider the following spaces and projections related to $\bX_0$ and $\bY_0$. We start with considering the low-rank component $\bX_0$. Let  $\mathcal{T}_0$ be the span of matrices with either the row space of $\bX$ are contained in that of $\bX_0$ or the column space of $\bX$ are contained in that of $\bX_0$:
\begin{align}\label{T0}
\mathcal{T}_0=&T_0(\bX_0)\cr=&\Big\{\bX_1+\bX_2: \in\mathbb{R}^{n\times m}:  \ \text{range}(\bX_1)\subseteq \text{range}(\bX_0), \cr &\ \ \text{range}(\bX_2^\T)\subseteq \text{range}(\bX_0^\T)\Big\}.
\end{align}
	Let $\mathcal{P}_{\mathcal{T}_0}$ be the orthogonal projector to $\mathcal{T}_0$. Under the inner product $\langle\bA,\bB\rangle=\tr (\bA^\T\bB)$, the projection is given by
\bel{projection0}
\mathcal{P}_{\mathcal{T}_0} (\bM)=\bU_0\bU_0^T\bM + \bM\bV_0\bV_0^\T - \bU_0\bU_0^\T\bM\bV_0\bV_0^\T.
\eel
	where $\bU_0\in\mathbb{R}^{N\times r}$ and $\bV_0\in\mathbb{R}^{M\times r}$ are the matrices of left and right orthogonal singular vectors corresponding to the nonzero singular values of $\bX_0$, and $r$ is the rank of $\bX_0$.
	
	Furthermore, let $\mathcal{T}$ be the span of matrices taking the form of $\bP\bX\bQ^\T$, with either the row space of $\bX$ are contained in that of $\bX_0$, or the column space of $\bX$ are contained in that of $\bX_0$:
	\bes
	\mathcal{T}&=&T(\bX_0;\bP,\bQ)\cr&=&\Big\{\bP(\bX_1+\bX_2)\bQ^\T: \in\mathbb{R}^{N\times M}: \ \text{range}(\bX_1)\subseteq \text{range}(\bX_0), \cr && \ \ \ \ \text{range}(\bX_2^\T)\subseteq \text{range}(\bX_0^\T)\Big\}.
	\ees
Apparently, $\mathcal{T}$ reduces to  $\mathcal{T}_0$ when $\bP$ and $\bQ$ are identity matrices. We further define the orthogonal projector onto $\mathcal{T}$ as $\mathcal{P}_T$:
	\bel{projection1}
	\mathcal{P}_\mathcal{T} (\bM)=\tbU_0\tbU_0^T\bM + \bM\tbV_0\tbV_0^\T - \tbU_0\tbU_0^\T\bM\tbV_0\tbV_0^\T,
	\eel
	where $\tbU_0\in\mathbb{R}^{N\times r}$ and $\tbV_0\in\mathbb{R}^{M\times r}$  are the left singular matrices of $\bP\bU_0$ and $\bQ\bV_0$, respectively. 
	Given such projections, we introduce a property that measures the sparseness of the singular vectors of $\bP\bX_0\bQ^\T$:
	\begin{align}\label{beta}
	\beta(\rho) = \rho^{-1}\|\tbU_0\tbU_0^\T\|_{\tvec(\infty)}+\rho  \|\tbV_0\tbV_0^\T\|_{\tvec(\infty)}+\|\tbU_0\|_{2\rightarrow \infty} \|\tbV_0\|_{2\rightarrow \infty}.
	\end{align}
	
	We shall note that the projection (\ref{projection1}) is equivalent to the following form:
	\bel{projection2}
	\mathcal{P}_\mathcal{T} (\bM)=\widebar{\bU}_0\widebar{\bU}_0^T\bM + \bM\widebar{\bV}_0\widebar{\bV}_0^\T - \widebar{\bU}_0\widebar{\bU}_0^T\bM\widebar{\bV}_0\widebar{\bV}_0^\T,
	\eel
	where $\widebar{\bU}_0\in\mathbb{R}^{N\times r}$ and $\widebar{\bV}_0\in\mathbb{R}^{M\times r}$ are, respectively, matrices of left and right orthogonal singular vectors corresponding to $\widetilde{\bX}_0=\bP\bX_0\bQ^\T$. In other words, (\ref{projection1}) and (\ref{projection2}) are equivalent in the sense that
	\bel{projection3}
	\tbU_0\tbU_0^\T=\widebar{\bU}_0\widebar{\bU}_0^T, \ \ \  \tbV_0\tbV_0^\T=\widebar{\bV}_0\widebar{\bV}_0^T.
	\eel
	Note that $\tbU$ and $\widebar{\bU}$ (or $\tbV$ and $\widebar{\bV}$) are not necessarily the same to hold (\ref{projection3}). Building on $\widebar{\bU}_0$ and $\widebar{\bV}_0$, $\beta(\rho)$ could be defined as
	\begin{align}\label{beta2}
	\beta(\rho) = \rho^{-1}\|\widebar{\bU}_0\widebar{\bU}_0^\T\|_{\tvec(\infty)}+\rho  \|\widebar{\bV}_0\widebar{\bV}_0^\T\|_{\tvec(\infty)}+\|\widebar{\bU}_0\|_{2\rightarrow \infty} \|\widebar{\bV}_0\|_{2\rightarrow \infty},
	\end{align}	
	due to (\ref{projection3}) and $\|\tbU_0\|_{2\rightarrow \infty}=\|\widebar{\bU}_0\|_{2\rightarrow \infty}$ and $\|\tbV_0\|_{2\rightarrow \infty}=\|\widebar{\bV}_0\|_{2\rightarrow \infty}$, which in fact is also a consequence of (\ref{projection3}). 
	We will mainly use the definition (\ref{projection1}) for the projection $\mathcal{P}_\mathcal{T} (\bM)$ in our following analysis as it allows us to ``separate'' the construction of $\tbU$ and $\tbV$ and brings us a lot of benefits when we bound the estimation errors later.

We define the following quantity to link the projections $\mP_{\mT^\perp}(\cdot)$ in (\ref{projection1}) and $\mP_{\mT_0^\perp}(\cdot)$ in (\ref{projection0}).
\begin{align}\label{eta1}
\eta_1=&\max\Big\{\eta:\ \eta>0, \  \cr &\ \ \ \ \ \  \eta\|\mP_{\mT^\perp}(\bP\bX\bQ^\T)\|_*\le \|\mP_{\mT_0^\perp}(\bX)\|_*, \  \forall \bX\in\mathbb{R}^{n\times m}
\Big\},
\end{align}
The existence of $\eta_1$ can be guaranteed through Proposition \ref{prop-eta1} below.
\begin{proposition}\label{prop-eta1}
	Let $\mathcal{P}_{\mathcal{T}} (\cdot)$ and $\mathcal{P}_{\mathcal{T}_0} (\cdot)$ be as in (\ref{projection1}) and  (\ref{projection0}). 
	Then, for any $ \bX\in\mathbb{R}^{n\times m}$,
	\bes
	\|\mP_{\mT_0^\perp}(\bX)\|_*=0 \ \ \Rightarrow  \ \ \|\mP_{\mT^\perp}(\bP\bX\bQ^\T)\|_*=0.
	\ees
\end{proposition}
\medskip

Now we consider the sparse component $\bY_0$. 
 Define the space of matrices whose supports are subsets of the supports of $\bY_0$:
	\bes
	\mS= S(\bY_0):=\{\bY\in\mathcal{R}^{N\times M}, \supp(\bY)\subseteq\supp(\bY_0)\}.
	\ees
	Define the orthogonal projector to $\mathcal{S}$ as $\mathcal{P}_{\mathcal{S}}$. Under the inner product $\langle A, B\rangle=\tr (A^\T B)$, this projection is given by
	\bel{projections}
	[\mathcal{P}_{\mS}(\bX)]_{i,j} = \left\{\begin{array}{cc}
		X_{i,j}, & (i,j)\in \supp (\bX_0),\\
		0, & \text{otherwise} ,
	\end{array}
	\right.
	\eel
	for $i=1,\cdots, N$ and $j=1,\cdots,M$. Furthermore, for any matrix $\bM$, define a $\|\cdot\|_{p\rightarrow q}$ transformation norm as
	\bes
	\|\bM\|_{p\rightarrow q} = \max\{\|\bM\nu\|_q:\nu\in\mathbb{R}^n, \|\nu\|_p\le 1\}.
	\ees
	Then we define the following property that measures the sparseness of $\bY_0$:
	\bel{alpha}
	\alpha(\rho) = \max\big\{\rho \|\sgn(\bY_0)\|_{1\rightarrow 1}, \  \rho^{-1}\|\sgn(\bY_0)\|_{\infty\rightarrow \infty}\big\},
	\eel
	where $\{\sgn(\bM)\}_{i,j}=\sgn(\bM_{i,j})$ is the sign of $M_{i,j}$, and $\rho>0$ is a parameter to accommodate disparity between the number of rows and columns with a natural choice of $\rho$ being $\rho=\sqrt{M/N}$. 
	As $\|\bM\|_{1\rightarrow 1}=\max_j\|\bM e_j\|_1$ and $\|\bM\|_{\infty \rightarrow \infty}=\max_i\|\bM^\T e_i\|_1$, $\|\sgn(\bY_0)\|_{1\rightarrow 1}$ and $\|\sgn(\bY_0)\|_{\infty\rightarrow \infty}$ respectively measures the maximum number of nonzero entries in any row and any column of $\bY_0$.
	This explains why $\alpha(\rho)$ is a quantity that measures the sparseness of $\bY_0$. 
	
	Now we introduce a quantity related to the projection $\mP_{\mS}(\cdot)$ in (\ref{projections}) and $\mP_{\mS^\perp}(\cdot)$, 
	\begin{align}\label{eta2} 
	\eta_2=&\max\Big\{ \eta: \ \eta>0,  \ \eta \|\bP^*\bY\bQ^*\|_{\tvec(1)} \cr &\ \ \ \ \ \ \ \ \ \  \ \ \ \  \ \le \eta\|\mP_{\mS}(\bY)\|_{\tvec(1)}+\|\mP_{\mS^\perp}(\bY)\|_{\tvec(1)},   \forall \bY\in\mathbb{R}^{N\times M}
	\Big\}.
	\end{align}
	The existence of $\eta_2$ is obvious. By replacing the $\tvec(1)$-norm in (\ref{eta2}) to the $\tvec(2)$-norm, we can have a rough idea about the scale of $\eta_2$.  As $\bP^*$ and $\bQ^*$ are projection matrices, we have
	\bes
	\|\bP^*\bY\bQ^*\|_{\tvec(2)} \le \|\bY\|_{\tvec(2)} \le \|\mP_{\mS}(\bY)\|_{\tvec(2)}+\|\mP_{\mS^\perp}(\bY)\|_{\tvec(2)}, \  \forall  \bY\in\mathbb{R}^{N\times M}.
	\ees
	As a consequence,
	\begin{align*}
	1\le &\max\Big\{ \eta: \ \eta>0,  \ \eta \|\bP^*\bY\bQ^*\|_{\tvec(2)} \cr &\ \ \ \ \ \ \ \ \ \  \ \ \ \ \le \eta\|\mP_{\mS}(\bY)\|_{\tvec(2)}+\|\mP_{\mS^\perp}(\bY)\|_{\tvec(2)},   \forall \bY\in\mathbb{R}^{N\times M}
	\Big\}.
	\end{align*}
	Although $\tvec(1)$-norm is used in (\ref{eta2}), an $\eta_2$ close to 1 can be expected for many combinations of $\bP^*$, $\bQ^*$ and  $\mP_{\mS}(\cdot)$.

\subsection{Main results}

To introduce our main results on recovering $\bP\bX_0\bQ^\T$ and $\bY_0$, we need the following properties related to $(\bU_0,\bV_0)$ and $(\bP,\bQ)$,
	\begin{align}\label{structure}
		\bGamma=&\big((\bP\bU_0)^+\big)^\T\bV_0^\T\bQ^++(\bP^+)^\T\bU_0(\bQ\bV_0)^+\cr &-\big((\bP\bU_0)^+\big)^\T(\bQ\bV_0)^+,
		\cr\gamma_1=&\|\bGamma\|_{\tvec(\infty)}, \ \ \gamma_2=\|\bGamma\|_{2\rightarrow 2}.
	\end{align}
	The quantity $\bGamma$ plays a key role in our analysis below. Note that when $\bP$ and $\bQ$ are identity matrices, $\bGamma=\bU_0\bV_0^\T$ and $\gamma_2=\|\bU_0\bV_0^\T\|_{2\rightarrow 2} =1$.
	
	Furthermore, define the following random error terms related to the noise matrix:
	\bel{error}
	\eps_{2\rightarrow 2}&=&\|\bE\|_{2\rightarrow 2},
	\cr \eps_\infty&=&\|\mP_{\mT}(\bE)\|_{\tvec(\infty)} + \|\bE\|_{\tvec(\infty)},
	\cr \eps'_\infty&=&\|\mP_{\mT}(\bP^*\bE\bQ^*)\|_{\tvec(\infty)} + \|\bP^*\bE\bQ^*\|_{\tvec(\infty)},
	\cr \eps_*&=&\|\mP_{\mT}(\bP^*\bE\bQ^*)\|_*,
	\eel
	where for any matrix $\bM$, $\bM^*$ is the projection matrix onto the column space of $\bM$. When $\bM$ is of full-column rank, $\bM^*=\bM(\bM^\T\bM)^{-1}\bM^\T$. Given these error terms, we suppose that the penalty levels $\lam_1$ and $\lam_2$ satisfy the condition below for certain $c>1$ and $\rho>0$,
	\begin{align}
		&\alpha(\rho)\beta(\rho) < 1\label{lambda1}\\
		& \left[ \sigma_{\max}^{-1}(\bP) \sigma_{\max}^{-1}(\bQ)-\frac{c\gamma_1\alpha(\rho)}{1-\alpha(\rho)\beta(\rho)}\right]\lam_1\label{lambda2}\\ &\  \ \ge c\left(\frac{\alpha(\rho)}{1-\alpha(\rho)\beta(\rho)}\lam_2+\frac{\alpha(\rho)}{1-\alpha(\rho)\beta(\rho)}\eps_{\infty} +\eps_{2\rightarrow 2}\right),
		\cr & \left[1-(1+c)\alpha(\rho)\beta(\rho)\right]\lam_2\ge c\left(\gamma_1\lam_1+ (2-\alpha(\rho)\beta(\rho))\eps_{\infty}\right), \label{lambda3}
	\end{align}
	We note that when $\bP$ and $\bQ$ are interpolation matrices with appropriate dimension, e.g., $N\ge 20$, $M\ge 20$, we have $\sigma_{\max}(\bP)\approx \sigma_{\max}(\bQ)\approx 1.53$, and $\sigma_{\min}(\bP)\approx \sigma_{\min}(\bQ)\approx 1.00$.

	Finally, we define $\delta_1$, $\delta_2$ and $\delta$ as functions of $r$, $s$, $\gamma_1$, $\gamma_2$, $\alpha(\rho)$, $\beta(\rho)$, $\lam_1$, $\lam_2$ and the error terms. These quantities will be used in Theorem \ref{thm-main} below.
	\begin{align}\label{delta}
		\delta_1=&r\Big(\frac{2\alpha(\rho)}{1-\alpha(\rho)\beta(\rho)}(\lam_2+\gamma_1\lam_1+ \eps_{\infty})+2\eps_{2\rightarrow 2} +\lam_1\gamma_2\Big),
		\cr\delta_2=&\frac{s}{1-\alpha(\rho)\beta(\rho)}(\lam_2+ \lam_1\gamma_1+\eps_{\infty}),
		\cr \delta=& (\lam_1\gamma_2 +\eps_{2\rightarrow 2})\delta_1 + (\lam_2+\eps_\infty)\delta_2.
	\end{align}
		Now we are ready to state our main results.
		
	\begin{theorem}\label{thm-main}
		Let $r=|\rank(\bX_0)|$ and $s=|\supp(\bY_0)|$. 
		Let error terms $\eps_{2\rightarrow 2}$, $\eps_{\infty}$, $\eps'_{\infty}$, $\eps_{*}$ be as in (\ref{error}), $\gamma_1$ and $\gamma_2$ be as in (\ref{structure}) and $\delta$ be as in (\ref{delta}). Further let $\eta_1, \eta_2$ be as in (\ref{eta1}), (\ref{eta2}) and $\eta_0=\min\left(\eta_2, \ \eta_1\sigma_{\max}(\bP)\sigma_{\max}(\bQ)\right)$. Assume that $\bP$ and $\bQ$ are of full column rank. Then, when (\ref{lambda1}) to (\ref{lambda3}) hold for some $\rho>0$ and $c>1$, we have
		\bel{thm-main-1}
		&&(1-\alpha(\rho)\beta(\rho))\|\bP^*(\hbY-\bY_0)\bQ^*\|_{\tvec(1)}
		\cr &\le &[\lam_2(1-1/c)\eta_0]^{-1} \delta + 5\lam_2 s +2s\eps_{\infty}+3s\eps'_{\infty} \cr &&+2\sigma_{\min}^{-1}(\bP)\sigma^{-1}_{\min}(\bQ)\lam_1 \sqrt{sr},
		\eel	
		\bel{thm-main-2}
		&&(1-\alpha(\rho)\beta(\rho))\|\hbY-\bY_0\|_{\tvec(1)}
		\cr &\le &[2(1-1/c)\lam_2]^{-1}(1+\eta_0^{-1})\delta+5\lam_2 s +2s\eps_\infty\cr &&+3s\eps'_{\infty}+2\sigma_{\min}^{-1}(\bP)\sigma^{-1}_{\min}(\bQ)\lam_1 \sqrt{sr},
		\eel
		and	 
		\bel{thm-main-3}
		&& \|\bP(\hbX-\bX_0)\bQ^\T\|_*
		\cr &\le&  [2(1-1/c)\lam_1\eta_1]^{-1}\delta +\eps_*+2\sigma_{\min}^{-1}(\bP)\sigma^{-1}_{\min}(\bQ)\lam_1 r
		\cr &&+  \sqrt{2r}\|\bP^*(\hbY-\bY_0)\bQ^*\|_{\tvec(2)}. 
		\eel
	\end{theorem}
	\medskip
	
	We note that the last term in the RHS of (\ref{thm-main-3}) can be easily bounded by $\sqrt{2r}\|\bP^*(\hbY-\bY_0)\bQ^*\|_{\tvec(1)}$ and then (\ref{thm-main-1}) can be applied. 
	To understand the derived bounds in Theorem \ref{thm-main}, we first recall that the matrices $\bP$ and $\bQ$ are of full column rank. When $\sigma_{\min}(\bP)\asymp\sigma_{\max}(\bP)\asymp\sigma_{\min}(\bQ)\asymp\sigma_{\max}(\bQ)\asymp  \mathcal{O}(1)$ as of interpolation matrices and 
	\bes
	\gamma_1\alpha(\rho)\asymp \gamma_2\asymp \mathcal{O}(1),
	\ees
	the penalty levels $\lam_1$ and $\lam_2$ of order
	\begin{align}\label{lambdaorder}
		\lam_1=&\mathcal{O}\Big([\alpha(\rho)\eps_{\infty}]\vee \eps_{2\rightarrow 2}\Big),
		\cr \lam_2 =& \mathcal{O}([1/\alpha(\rho)] \lam_1),
	\end{align}
	would satisfy conditions (\ref{lambda1}) and (\ref{lambda2}). As a consequence, the error bounds in Theorem \ref{thm-main} are of order
	\begin{align}\label{bound1}
		\|\hbY-\bY_0\|_{\tvec(1)}\asymp&\|\bP^*(\hbY-\bY_0)\bQ^*\|_{\tvec(1)} 
		\cr= &\mathcal{O}\Big(r\alpha(\rho)\big\{[\alpha(\rho)(\eps_{\infty} \vee \eps'_{\infty})]\vee \eps_{2\rightarrow 2}\big\}\Big),
	\end{align}
	and
	\begin{align}\label{bound2}
		&\|\bP(\hbX-\bX_0)\bQ^\T\|_{\tvec(1)}\cr =& \mathcal{O}\Big(r^{3/2}\alpha(\rho)\big\{[\alpha(\rho)(\eps_{\infty} \vee \eps'_{\infty})]\vee \eps_{2\rightarrow 2}\big\}+\eps_*\Big).
	\end{align}
	\citet{hsu2011robust} derived the upper bounds for $\|\hbY-\bY_0\|_{\tvec(1)}$ and $\|\hbX-\bX_0\|_{*}$ under the classical RPCA setup, i.e., $(\bP,\bQ)=(\bI_N,\bI_M)$. They imposed the constraint $\|\hbY-\bZ\|_{\tvec(\infty)}\le b$ in the optimization for some $b\ge \|\bY_0-\bZ\|_{\tvec(\infty)}$, while also allow $b$ to go to infinity. We note that the error bounds (\ref{bound1}) and (\ref{bound2}) is of the same order to their results when no knowledge of $b$ is imposed, i.e., $b=\infty$.
	In fact, Theorem \ref{thm-main} can be viewed as a generalization of \citet{hsu2011robust} for arbitrary full column rank matrices $\bP$ and $\bQ$.
	
	We still need to understand the random error terms in the bound. When the noise matrix $\bE$ has i.i.d. Gaussian entries, $E_{i,j}\sim \mathcal{N}(0,\sigma^2)$, by \citet{davidson2001local}, we have the following probabilistic upper bound,
	\bes
	\|\bE\|_{2\rightarrow 2}&\le& \sigma\sqrt{N} + \sigma\sqrt{M} +\mathcal{O}(\sigma),
	\cr \|\bP^*\bE\bQ^*\|_{2\rightarrow 2}&\le&
	\sigma\sqrt{N} + \sigma\sqrt{M} +\mathcal{O}(\sigma).
	\ees 
	In addition, for the terms with $\tvec(\infty)$-norm, we have the following inequalities hold with high probability
	\bes
	\|\bE\|_{\tvec(\infty)}&\le& \mathcal{O}(\sigma\log(MN)),
	\cr \|\mP_{\mT}(\bE)\|_{\tvec(\infty)}&\le& \mathcal{O}(\sigma\log(MN)),
	\cr \|\bP^*\bE\bQ^*\|_{\tvec(\infty)}&\le& \mathcal{O}(\sigma\log(MN)),
	\cr \|\mP_{\mT}(\bP^*\bE\bQ^*)\|_{\tvec(\infty)}&\le& \mathcal{O}(\sigma\log(MN)).
	\ees
	Finally, for the nuclear-normed error term, 
	\bes
	\|\mP_{\mT}(\bP^*\bE\bQ^*)\|_*\le 2r\|\bP^*\bE\bQ^*\|_{2\rightarrow 2}\le 2r \sigma\sqrt{N} +2r \sigma\sqrt{M} +\mathcal{O}(r\sigma)
	\ees  
	holds with high probability, where the first inequality holds by Lemma \ref{lm-4} in the supplementary material.
	Then we can summarize the asymptotic probabilistic bound below.
	\begin{align*}
		\|\hbY-\bY_0\|_{\tvec(1)}\asymp&\|\bP^*(\hbY-\bY_0)\bQ^*\|_{\tvec(1)} 
		\cr= &\mathcal{O}\Big(\sigma r\alpha(\rho)\big\{\left[\alpha(\rho)\log(MN)\right]\vee [\sqrt{N}+\sqrt{M}]\big\}\Big),
		\cr 	\|\bP(\hbX-\bX_0)\bQ^\T\|_{*}=&\mathcal{O}\Big(\sigma r^{3/2}\alpha(\rho)\big\{\left[\alpha(\rho)\log(MN)\right]\vee [\sqrt{N}+\sqrt{M}]\big\}\Big).
	\end{align*}
	We note that the bound on $\|\bP(\hbX-\bX_0)\bQ^\T\|_{*}$ can be improved if prior knowledge is known on the upper bound of $\|\bY_0\|_\infty$.
	\subsection{Outline of proof}\label{sec:proof}
	The key to prove Theorem \ref{thm-main} is the following two theorems. In Theorem \ref{thm-link}, we provide a transfer property between the two projections $\mP_{\mT^\perp}(\cdot)$ and $\mP_{\mT_0^\perp}(\cdot)$ through $\bGamma$. Building on the transfer property, we in Theorem \ref{thm-dual} construct a dual certificate $(\bD_{\mS}, \bD_{\mT})$ such that (1) $\bD_\mS+\bD_\mT+\bE$ is a subgradient of $\lam_2\|\bY\|_{\tvec(1)}$ at $\bY=\bY_0$, and (2) $\bP^\T(\bD_\mS+\bD_\mT+\bE)\bQ$ is a subgradient of $\lam_1\|\bX\|_{*}$ at $\bX=\bX_0$.
	
	\begin{theorem}[Transfer Property]\label{thm-link}
		Suppose $\bP$ and $\bQ$ are of full column rank. Let $\bGamma$ be as in (\ref{structure}). Let $\bD\in\mathbb{R}^{N\times M}$ be any matrix satisfies 
		\bes
		\mP_{\mT}(\bD) = \bGamma.
		\ees
		Then, $\bP^T\bD\bQ$ is a sub-gradient of $\|\bX\|_*$ at $\bX_0$, in other words, 
		\bes
		\mP_{\mT_0}(\bP^\T\bD\bQ) = \bU_0\bV_0^\T.
		\ees
	\end{theorem}
	
	\begin{theorem}[Dual Certificate]\label{thm-dual}
		Let $r=\rank(\bX_0)$, $s=\|\bY\|_0$ and $\rho>0$.
		Let error terms $\eps_{2\rightarrow 2}$, $\eps_{\infty}$, $\eps'_{\infty}$, $\eps_{*}$ be as in (\ref{error}) and $\eta_1, \eta_2$ be as in (\ref{eta1}), (\ref{eta2}), respectively. Assume that $\inf_{\rho>0}\alpha(\rho)\beta(\rho) <1$ and the penalty level $\lam_1$ and $\lam_2$ satisfy (\ref{lambda2}) and (\ref{lambda3}) for some $c>1$. Suppose $\bP$ and $\bQ$ are of full column rank. 
		Then, the following quantity $\bD_\mS$ and $\bD_\mT$ are well defined, 
		\bes
		\bD_\mS  &=& (\mI-\mP_\mS\circ\mP_\mT)^{-1}\left(\lam_2\sgn(\bY_0)-\lam_1\mP_{\mS}(\bGamma)- (\mP_\mS\circ\mP_{\mT^\perp})(\bE)\right),
		\cr \bD_\mT &=&(\mI-\mP_\mT\circ\mP_\mS)^{-1}\left(\lam_1\bGamma-\lam_2\mP_{\mT}\big(\sgn(\bY_0)\big)-(\mP_\mT\circ\mP_{\mS^\perp})(\bE)\right).
		\ees
		They satisfy
		\bel{dual-1}
		&&\mP_\mS(\bD_\mS+\bD_\mT+\bE) =\lam_2\sgn(\bY_0),
		\cr &&\mP_\mT(\bD_\mS+\bD_\mT+\bE) =\lam_1\bGamma,
		\cr &&\mP_{\mT_0}\left(\bP^\T(\bD_\mS+\bD_\mT+\bE)\bQ\right) =\lam_1\bU_0\bV_0^\T,
		\eel
		and
		\bel{dual-2}
		&\|\mP_{\mS^\perp}(\bD_\mS+\bD_\mT+\bE)\|_{\tvec(\infty)} &\le\lam_2/c,
		\cr &\|\mP_{\mT_0^\perp}\left(\bP^\T(\bD_\mS+\bD_\mT+\bE)\bQ\right)\|_{2\rightarrow 2} &\le \lam_1/c.
		\eel
		Moreover,
		\begin{align}\label{dual-3}
			&\|\bD_\mS\|_{2\rightarrow 2}\le \frac{\alpha(\rho)}{1-\alpha(\rho)\beta(\rho)}(\lam_2+\gamma_1\lam_1+ \eps_{\infty}),
			\cr &\|\bD_\mT\|_{\tvec(\infty)}\le \frac{1}{1-\alpha(\rho)\beta(\rho)} \left(\gamma_1\lam_1+\lam_2 \alpha(\rho)\beta(\rho)+\eps_{\infty}\right),
			\cr &\|\bD_\mT\|_{*}\le r\left(\frac{2\alpha(\rho)}{1-\alpha(\rho)\beta(\rho)}(\lam_2+\gamma_1\lam_1+ \eps_{\infty})+2\eps_{2\rightarrow 2}+\lam_1\gamma_2\right),
			\cr &\|\bD_\mS\|_{\tvec(1)}\le  \frac{s}{1-\alpha(\rho)\beta(\rho)}(\lam_2+ \lam_1\gamma_1+\eps_{\infty}),
			\cr & \|\bD_\mT+\bD_\mS\|_2^2\le  (\lam_2+\eps_\infty)  \|\bD_\mS\|_{\tvec(1)}+ (\lam_1\gamma_2+\eps_{2\rightarrow 2}) \|\bD_\mT\|_{*}.
		\end{align}
	\end{theorem}

	\section{Simulation studies}\label{sec:sim}
	In this section, we conduct a comprehensive simulation study to demonstrate the performance of PRPCA.
	Without loss of generality, all the simulations are for square matrix recovery, i.e., $M=N$. 
	
We consider the model
	\bel{sim-1}
	\bZ = \bP_0\bX_0\bQ_0^\T + \bY_0 +\bE
	\eel
under two cases
		\begin{itemize}
		\item[(Case 1)] $\bP_0$ and $\bQ_0$ are the interpolation matrices, i.e., $\bP_0=\bQ_0=\bJ_N$
		\item[(Case 2)]  $\bP_0$ and $\bQ_0$ are the generalized interpolation matrices in (\ref{def-3-1}) with all the weights $w_{ur}^i$, $w_{lr}^i$ are i.i.d. generated from $\mathcal{N}(0.5,0.2^2)$. 
	\end{itemize}
We study the the following optimization problem
	\begin{equation}\label{sim-2}
	(\hbX, \hbY)\in \argmin_{\substack{\bX\in \mathbb{R}^{n(\bP)\times m(\bQ)}\\ \bY\in  \mathbb{R}^{N \times M}}} \frac{1}{2} \left\|\bZ-\bP\bX\bQ^\T-\bY \right\|_F^2 +\lambda_1 \|\bX\|_* +\lambda_2\|\bY\|_{\tvec(1)},
	\end{equation} 
	with four sets of $(\bP, \bQ)$, where $n(\bP)$ and $ m(\bQ)$ refers to the number of columns of $\bP$ and $\bQ$, respectively.
	\begin{itemize}
		\item  The standard RPCA, with both $\bP$ and $\bQ$ being identity matrices, denoted as ``no interpolation''. 
		With such $\bP$ and $\bQ$, apparently there exists $\bX$ such that $\bP\bX\bQ^\T=\bP_0\bX_0\bQ_0^\T$ under both cases. In other words, the optimization problem (\ref{sim-2}) is correctly specified under both cases.
		
		\item Both $\bP$ and $\bQ$ are interpolation matrices, i.e., $\bP=\bQ=\bJ_N$, denoted as ``single interpolation''. With such $\bP$ and $\bQ$, the optimization problem (\ref{sim-2}) is correctly specified under Case 1, while mis-specified under Case 2.
		\item Both $\bP$ and $\bQ$ are estimated based on robust linear regression (LR) with Huber Loss described in Section \ref{sec:lrsmooth}, denoted as ``LR interpolation''.  With such $\bP$ and $\bQ$, the optimization problem (\ref{sim-2}) is mis-specified under both cases as the estimated $\hbP$ and $\hbQ$ may not recover $\bP_0$ and $\bQ_0$ exactly with probability goes to 1. 
		\item  Both $\bP$ and $\bQ$ are double interpolation matrices, i.e., $\bP=\bQ=\bJ_N\times \bJ_{N/2}$, denoted as ``double interpolation''. With such $\bP$ and $\bQ$, the optimization problem (\ref{sim-2}) is mis-specified under both cases.
	\end{itemize}



Under both cases, we generate each entry of the noise term $\bE$ from an i.i.d $\mathcal{N}(0,\sigma^2)$ distribution. 
	The low-rank matrix $\bX_0$ is generated as $\bX_0=\bU_0\bV_0^\T$, where both $\bU_0$ and $\bV_0$ are $N\times r$ matrices with i.i.d. $\mathcal{N}(0,\sigma_0^2)$ entries. Each entry of the sparse component $\bY_0$ is i.i.d. generated, and being
	0 with probability $1-\rho_s$, and uniformly distributed in $[-5,5]$ with  probability $1-\rho_s$.
	The simulation is run over a grid of values for the matrix dimension $N$, noise level $\sigma$ and sparsity level $\rho_s$:
	\begin{itemize}
		\item $\sigma=0.2, 0.4,0.6,0.8,1$
		\item $\rho_s=0.05,0.1,0.15,0.2,0.25$
		\item $N=60,100,200,300,400$
	\end{itemize}
The other parameters are fixed at at $r=10$ and $\sigma_0=0.6$ if otherwise specified.
	For all four sets of $(\bP,\bQ)$, we use the same penalty level with $\lam_1=\sqrt{2N}\sigma$ and $\lam_1=\sqrt{2}\sigma$. This penalty level are commonly used in RPCA with noise, for example, in \citet{zhou2010stable}. When $\bP$ and $\bQ$ are single or double interpolation matrices, other carefully tuned penalty levels may further increase the estimation accuracy. In other words, this penalty level setup may not favor the PRPCA with interpolation matrices. But it allows us to better tell the effects of $(\bP,\bQ)$ on the matrix recovery accuracy.
	

	We report the root mean square errors (RMSE) of recovering $\bP_0\bX_0\bQ_0^\T$, $\bY_0$ and $\bTheta$ with different choices of $(\bP,\bQ)$:
	\bes
	\text{RMSE}(\bP\bX\bQ^\T)&=&\frac{\left\|\bP\hbX\bQ^\T-\bP_0\bX_0\bQ_0^\T\right\|_F}{\sqrt{N*M}},
	\cr \text{RMSE}(\bY)&=&\frac{\|\hbY-\bY_0\|_F}{\sqrt{N*M}},
	\cr \text{RMSE}(\bTheta)&=&\frac{\|\hbTheta-\bTheta\|_F}{\sqrt{N*M}}.
	\ees
	Finally, we report the required computation time (in seconds; all calculations were performed on a 2018 MacBook Pro laptop with 2.3 GHz Quad-Core Processor and 16GB Memory).
	\subsection{Case 1 Study: Effect of noise level $\sigma$ }
	Figure \ref{fig-sigma} below reports the performance of PRPCA and RPCA over different noise levels under Case 1, with $\sigma_0=\sigma$ and fixed  $\rho_s=0.1$, $N=M=200$ and $r=10$. 

	We first observe that the PRPCA with both single interpolation and LR interpolation demonstrate clear advantages in recovering all three targets: $\bP_0\bX_0\bQ_0^\T$, $\bY_0$ and $\bTheta$. Particularly, PRPCA with single interpolation performs the best across the whole range of $\sigma$. While for the LR interpolation based PRPCA, we need to estimate the weights in the interpolation matrix first and then recover $\hbX$ and $\hbY$. It still outperforms RPCA in recovering $\bTheta$ across the whole range of $\sigma$, and in recovering $\bP\bX\bQ^\T$ and $\bY$ with relatively large $\sigma$ ($\sigma\ge 0.6$).
	In terms of PRPCA with the mis-specified double interpolation matrix, its overall performance is not as good as the RPCA. But note that when the noise level is small, it achieves similar recovery accuracy in recovering $\bY_0$ and $\Theta$ compared to RPCA. 
	
	Regarding the computation time,  it is clear that imposing interpolation matrices expedite the computation, and such improvement is significant.

	\begin{figure}[H]
		\centering
		\includegraphics[width=0.75\columnwidth]{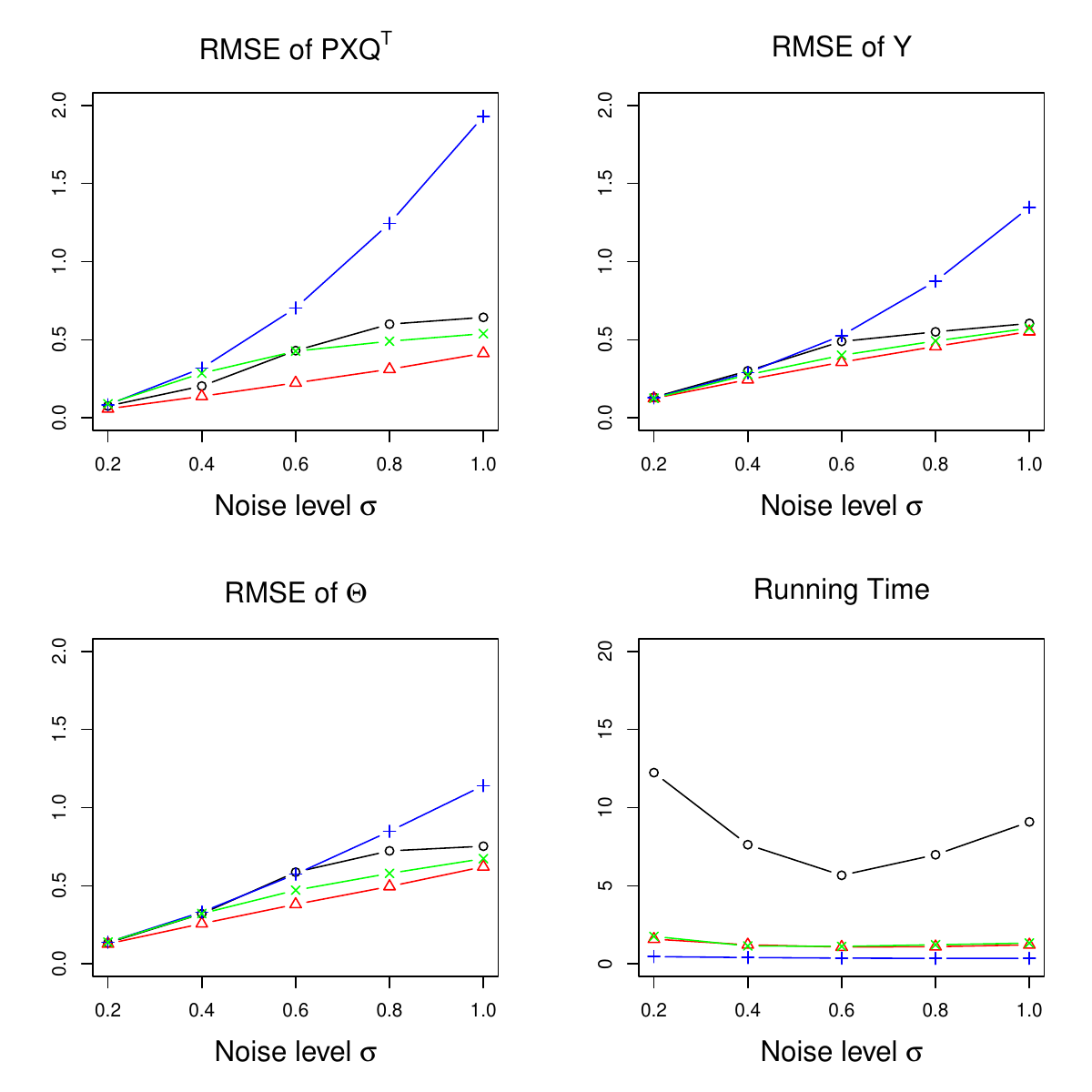}
		\caption{RMSE and running time with different $(\bP,\bQ)$ ranges over different $\sigma$ under Case 1. $\rho_s=0.1$, $N=M=200$, $r=10$. 
			Here: $\circ$ refers to no interpolation,  \lfR{$-\triangle-$} refers to single interpolation,  {\color{Green}{$-\times-$}} refers to LR interpolation, \lfB{$-+-$} refers to double interpolation. The running times are in seconds. }
		\label{fig-sigma}
	\end{figure}

\subsection{Case 1 Study: Effect of sparsity level $\rho_s$}
Figure \ref{fig-rho} reports the performance of PRPCA and RPCA over different $\rho_s$, the sparsity level of $\bY_0$, under Case 1, with fixed  $\sigma=0.6$, $N=M=200$ and $r=10$.

It is clear that both the single interpolation and LR interpolation based PRPCA demonstrate clear advantages in recovering $\bP_0\bX_0\bQ_0^\T$, $\bY_0$ and $\bTheta$. Indeed, both of them outperforms RPCA across the whole range of $\rho_s$ in estimating all three targets. 
Moreover, the PRPCA with mis-specified double interpolation also achieves smaller RMSE compared to RPCA in recovering $\bTheta$ when $\rho_s$ is relatively large, e.g., $\rho_s\ge 0.15$. One possible explanation for such phenomena is that when $\rho_s$ goes up, the mis-modeled entries in $\bP\bX\bQ^\T$ are more likely to be modeled by the sparse component $\bY$, thus further level up the performance of PRPCA. For the running time, we also see a significant speed up when interpolation matrices imposed.
	\begin{figure}[H]
		\centering
		\includegraphics[width=0.75\columnwidth]{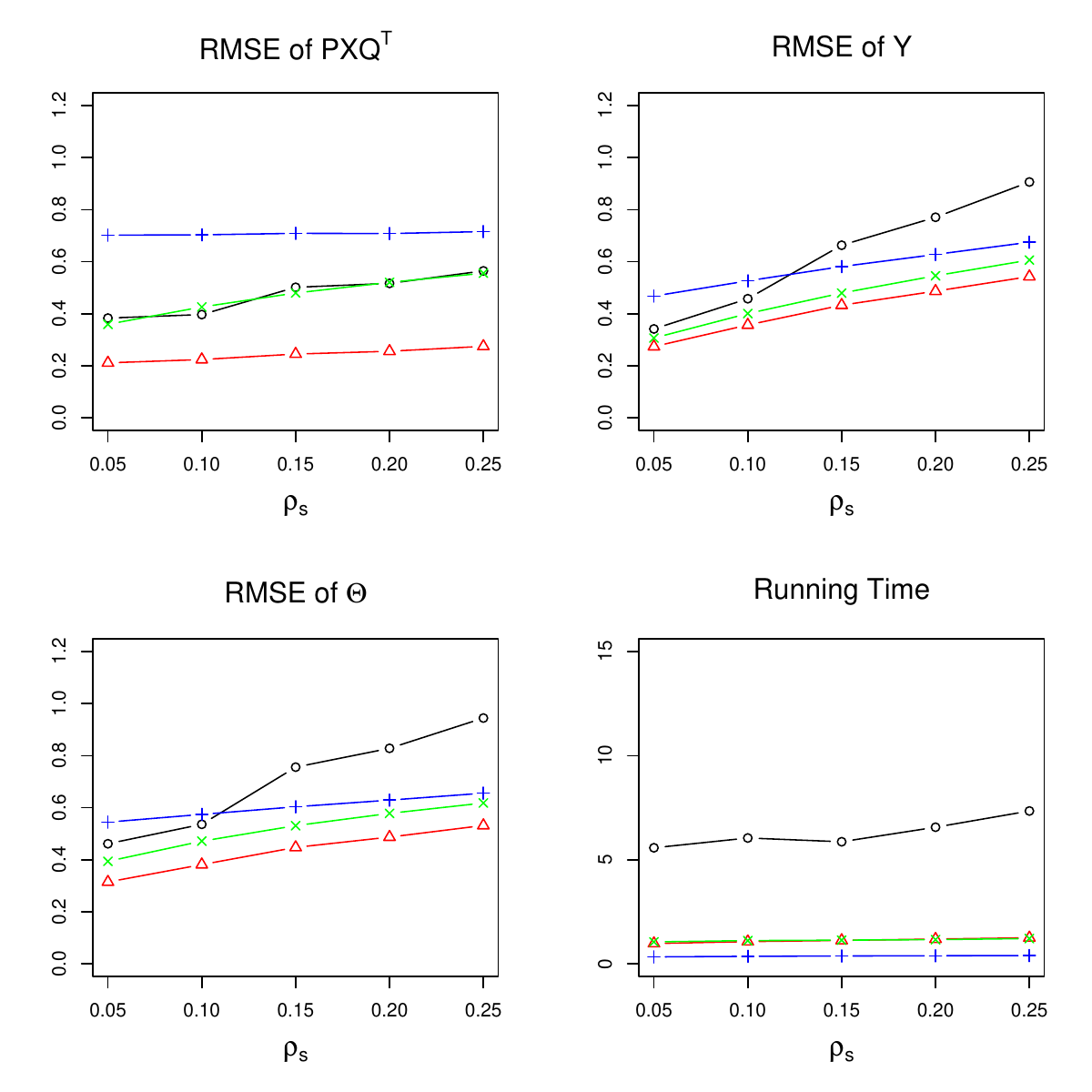}
		\caption{RMSE and running time with different $(\bP,\bQ)$ ranges over different  $\rho_s$ under Case 1. The sparsity of $\bY_0$. $\sigma=0.6$, $N=M=200$, $r=10$. 
			Here: $-\circ -$ refers to no interpolation,  \lfR{$-\triangle-$} refers to single interpolation,  \lfB{$-+-$} refers to double interpolation, {\color{Green}{$-\times-$}} refers to LR interpolation.
			The running times are in seconds. }
		\label{fig-rho}
	\end{figure}

	\subsection{Case 1 Study: Effect of matrix dimension $N$}
Figure \ref{fig-N} reports the performance of PRPCA and RPCA over different matrix dimension $N$ under Case 1. 
The noise level and sparsity level of $\bY_0$ are fixed at $\sigma=0.6$ and $\rho_s=0.05$.

We again see that the advantage of PRPCA with single and LR interpolation in recovering $\bP_0\bX_0\bQ_0^\T$, $\bY_0$ and $\bTheta$ across all the values of $N$.
In addition, the PRPCA with double interpolation also outperforms RPCA in recovering $\bY_0$ and $\bTheta$ when the matrix is of high dimension, e.g., $N\ge 200$. 
In terms of computation, the running time of RPCA grows almost exponentially as $N$ increase. The computational benefits of applying PRPCA is even more significant for high-dimensional matrix problems.

Under Case 1, the PRPCA with single interpolation is expected to outperform LR interpolation as the $\bP$ and $\bQ$ are correctly specified. We now consider Case 2, under which the $\bP_0$ and $\bQ_0$ are generated with noise. This allows us to test the performance of PRPCA when $\bP$ and $\bQ$ are not perfectly specified.

	\begin{figure}[H]
		\centering
		\includegraphics[width=0.75\columnwidth]{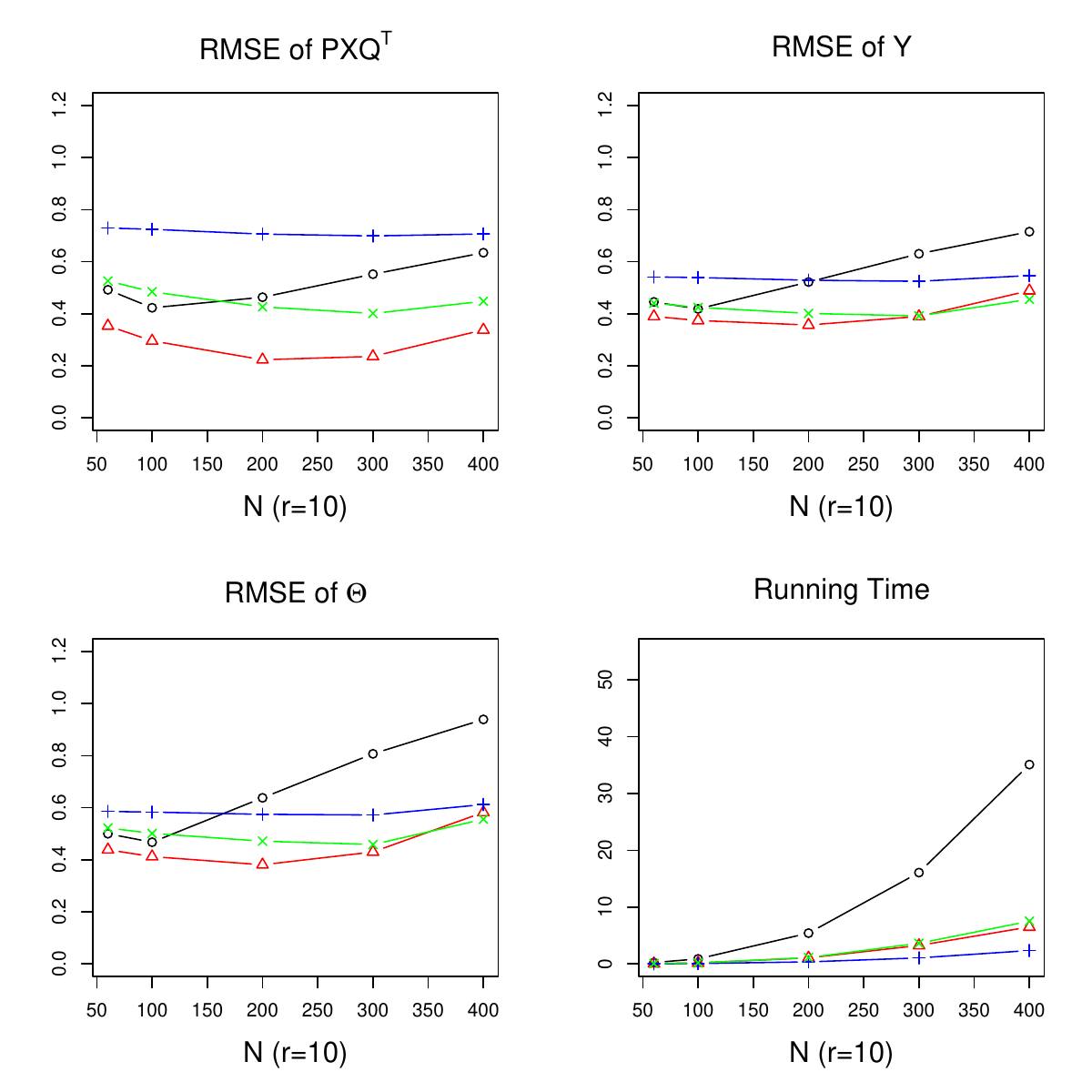}
		\vspace{-0.1in}
		\caption{RMSE and running time with different $(\bP,\bQ)$ ranges over different  $N$ and $r$. $\sigma=0.6$, $\rho_s=0.1$, $M=N$. In the first four plots, $r$ is fixed at 10, while in the second four plots, $r=0.05N$. 
				Here: $\circ$ refers to no interpolation,  \lfR{$-\triangle-$} refers to single interpolation,  {\color{Green}{$-\times-$}} refers to LR interpolation, \lfB{$-+-$} refers to double interpolation. The running times are in seconds. }
		\label{fig-N}
	\end{figure}

	\subsection{Case 2 Study: Effect of mis-specified $\bP_0$ and $\bQ_0$ across $N$}
	
	Figure \ref{fig-N-noise} reports the performance of PRPCA and RPCA when $\bP_0$ and $\bQ_0$ are generated with noise across $N$. 
	The noise level and sparsity level of $\bY_0$ are fixed at $\sigma=0.2$ and $\rho_s=0.05$. 
	
	Under Case 2, we observe that the performance of PRPCA with  LR interpolation demonstrate clear advantages over other approaches. Such advantage is resulted from the fact that the LR interpolation is data-dependent and able to achieve a better estimation of $(\bP_0,\bQ_0)$. While for the PRPCA with single and double interpolation matrices, they use a mis-specified $(\bP,\bQ)$ but still demonstrate robust performances, especially for recovering $\Theta$.
	Indeed, although the PRPCA with mis-specified smoothing mechanism may not perform as good as RPCA in recovering $\bP\bX\bQ^\T$ and $\bY$, but they in general achieve an better accuracy in  recovering $\bTheta$. An explanation for such phenomena is that the PRPCA could decompose $\bTheta$ differently with the sparse component accounting for more signals. As a consequence, the target $\bTheta$ can be recovered well even with mis-specified smoothing matrices.
	In terms of computation, we see a similar pattern as before: imposing the interpolation matrices is able to expedite computation significantly.

		\begin{figure}[H]
		\centering
		\includegraphics[width=0.75\columnwidth]{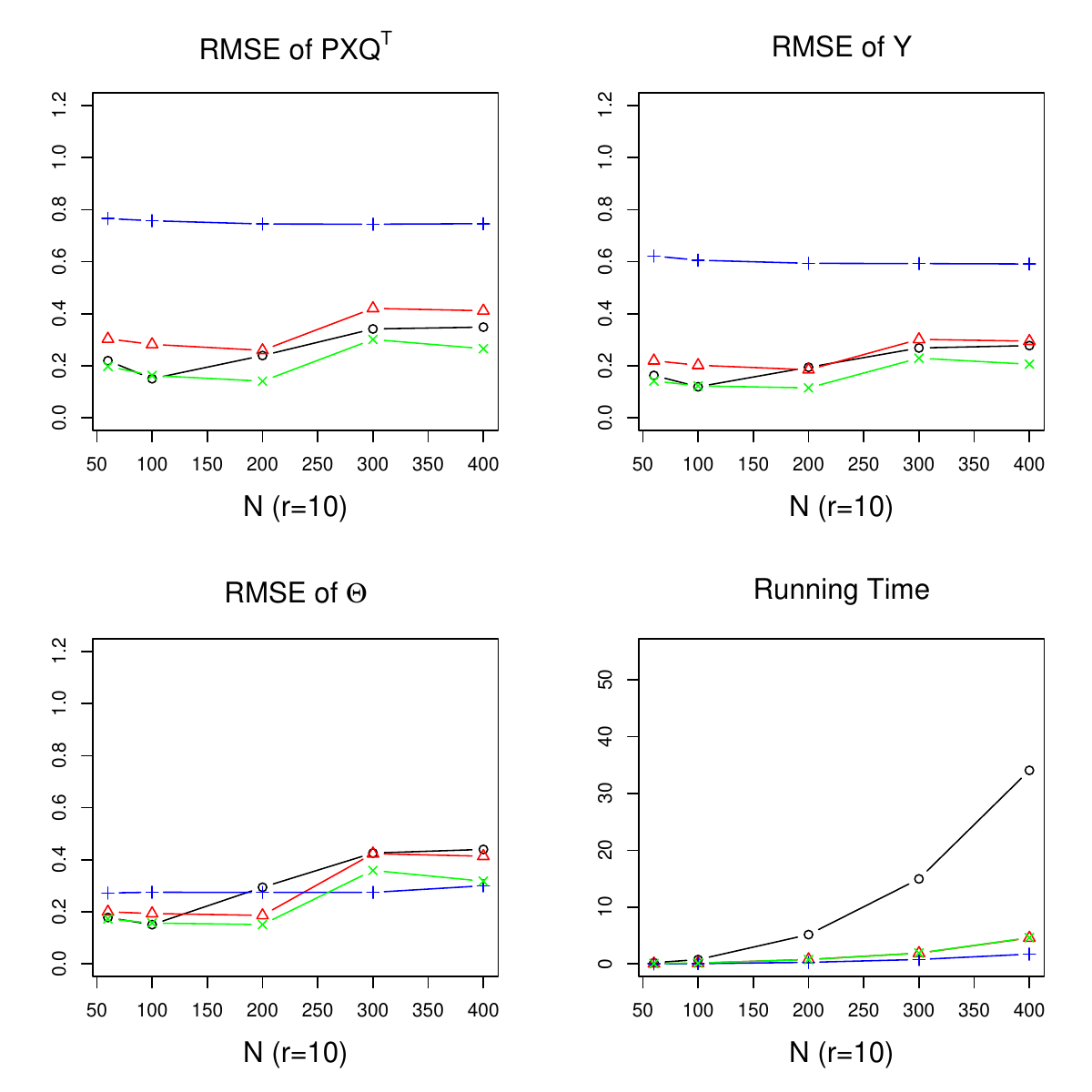}
		\vspace{-0.1in}
		\caption{RMSE and running time with different $(\bP,\bQ)$ ranges over different  $N$ and $r$. $\sigma=0.6$, $\rho_s=0.05$, $M=N$, $r=10$.	Here: $\circ$ refers to no interpolation,  \lfR{$-\triangle-$} refers to single interpolation,  {\color{Green}{$-\times-$}} refers to LR interpolation, \lfB{$-+-$} refers to double interpolation. The running times are in seconds. }
		\label{fig-N-noise}
	\end{figure}
		
After all, we conclude that the PRPCA with interpolation matrices performs consistently well across a large range of noise level, matrix dimension, sparsity of $\bY_0$. 
Moreover, even with mis-specified smoothing matrices, the PRPCA is still able to produce robust estimation, 
especially for recovering the mean matrix $\bTheta$.


	\section{The Lenna image analysis}\label{sec:lenna}
	In this section, we analyze the image of Lenna, a benchmark in image data analysis, and demonstrate the advantage of PRPCA with interpolation matrices over the RPCA. We consider the gradyscale Lenna image, which is of dimension $512\times 512$ and can be found at  \href{https://www.ece.rice.edu/~wakin/images/} {https://www.ece.rice.edu/~wakin/images/}. The image is displayed in Figure \ref{lenna} below.
	\begin{figure}[H]
		\begin{center}
			\includegraphics[width=0.45\columnwidth]{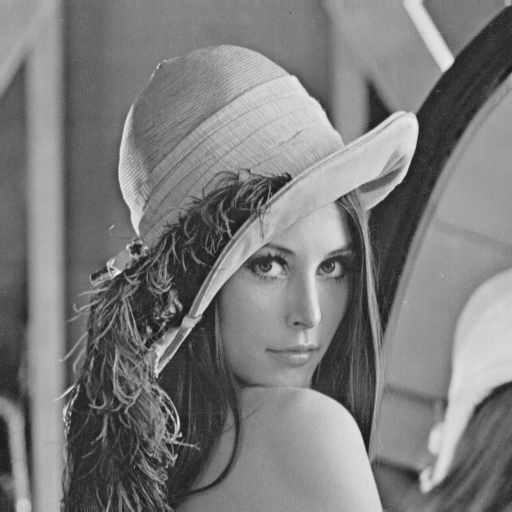}
			\caption{The gray-scale Lenna image}
			\label{lenna}
		\end{center}
		\vskip -0.2in
	\end{figure}
	
	
	We re-scale the Lenna image such that each pixel of the image is range from 0 to 1, with 0 represents pure black and 1 represents pure white. 
	Our target is to recover the Lenna image from its noisy version with different noise levels. That is, we observe
	\bes
	\bZ=\bTheta+\bE,
	\ees
	where $\bTheta$ is the true Lenna image and $\bE$ is the noise term with i.i.d entries generated from $\mathcal{N}(0,\sigma^2)$. We consider the noise levels range from 0.05 to 0.25. Specifically, we let $\sigma=0.05,0.1,0.15,0.2,0.25$. Figure \ref{lenna_blur} plots the Lenna image with different noise level.
	\begin{figure}[ht]
		\begin{center}
			\includegraphics[width=\columnwidth]{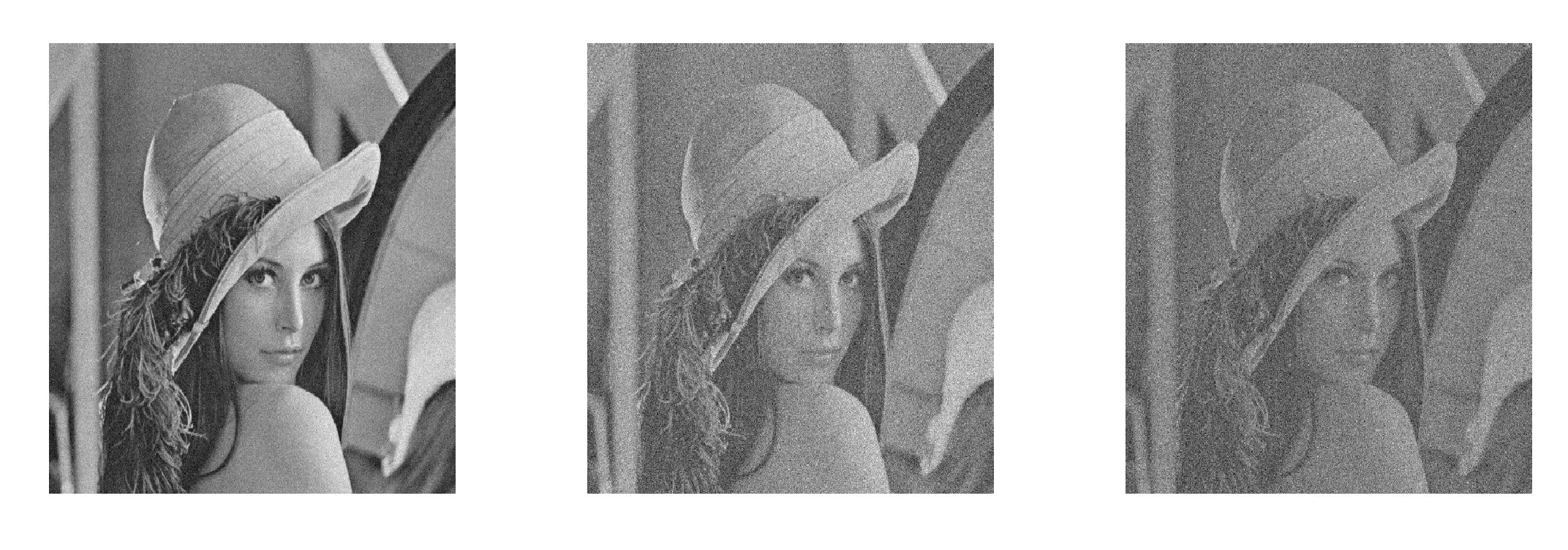}
			\vspace{-0.2in}
			\caption{The Lenna image with different noise levels, $\sigma=0.05$ (left), $\sigma=0.15$ (middle), $\sigma=0.25$ (right).}
			\label{lenna_blur}
		\end{center}
	\end{figure}

	As in the simulation study, we recover the image with three sets of $(\bP,\bQ)$: 1) identity matrices; 2) single interpolation matrices; 3) LR interpolation; and 4) double interpolation matrices. 
	
	The penalty levels are still fixed at $\lam_1=\sqrt{2N}\sigma$ and $\lam_1=\sqrt{2}\sigma$ for all three sets of $(\bP,\bQ)$. 
	As the true low-rank component $\bP\bX_0\bQ^\T$ and sparse component $\bY_0$ are not available in the real image analysis, we only measure the RMSE of $\bTheta$ and the computation time. We generate 100 independent noise terms $\bE$ and report the mean running time and RMSE in Figure \ref{lenna_sigma}. In addition, we plot the recovered Lenna image with different noise level in one implementation in Figure \ref{lenna_blur}.

	\begin{figure}[ht]
		\begin{center}
			\includegraphics[width=\columnwidth]{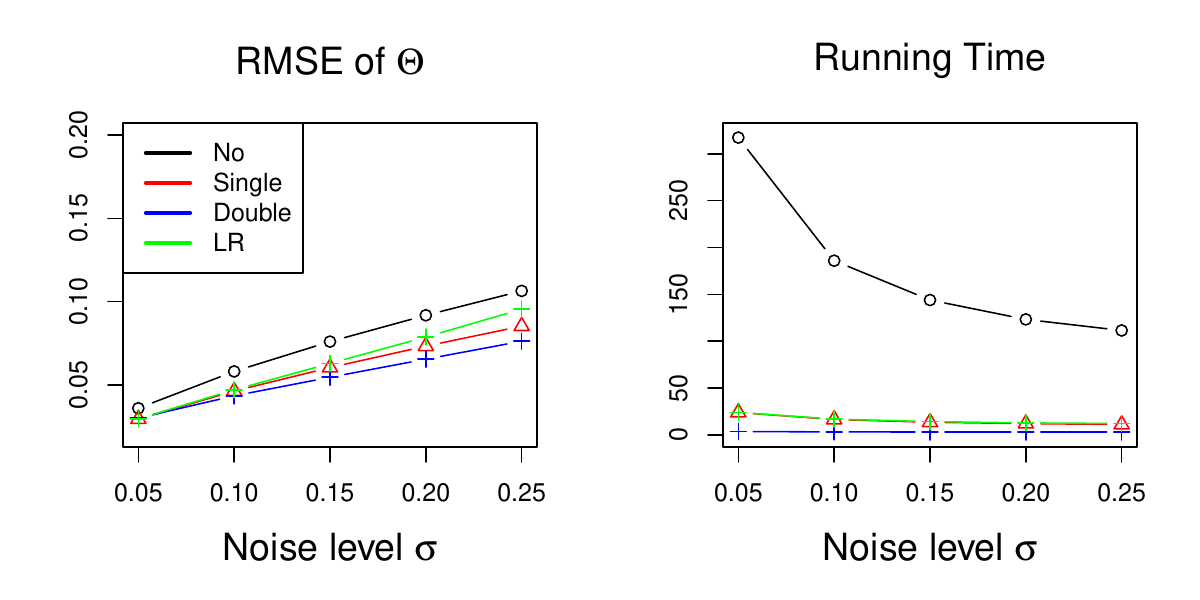}
			\vspace{-0.2in}
			\caption{The RMSE of $\bTheta$ and running time of the Lenna image analysis with different $(\bP,\bQ)$ and different $\sigma$. 
					Here: $\circ$ refers to no interpolation,  \lfR{$-\triangle-$} refers to single interpolation,  {\color{Green}{$-\times-$}} refers to LR interpolation, \lfB{$-+-$} refers to double interpolation. The running times are in seconds. }
			\label{lenna_sigma}
		\end{center}
	\end{figure}

	From Figure \ref{lenna_sigma} and Figure \ref{lenna_blur}, it is clear that the three PRPCA approaches outperform the RPCA significantly in terms of both image recovery accuracy and computation time across the whole range of $\sigma$. 
	Recall that the Lenna image is of dimension $512\times 512$. Under such dimension, the computational benefits of PRPCA is even more significant. 
	The PRPCA with single or LR interpolation is on average 10 times faster than RPCA, while PRPCA with double interpolation is at least 30 times faster than RPCA. In extreme case, when the noise level is low, e.g., $\sigma=0.05$, the average running time of PRPCA with double interpolation is 3.6 seconds. While RPCA requires 311.7 second, more than 86 times of that of PRPCA. 
	
	In terms of recovery accuracy, we see that the PRPCA with double interpolation even outperforms PRPCA with single or LR interpolation across the whole range of $\sigma$. In the simulation study, we conclude that when the target matrix is of large dimension and the sparsity of $\bY_0$, $\rho_s$, is high, the PRPCA with double or even more interpolation matrices would work well in terms of mean matrix $\bTheta$ recovery. The Lenna image can be viewed as such kind, with resolution $512\times 512$ and although unknown, but potentially large $\rho_s$. Thus it is not supervised to see the outstanding performance of double interpolated PRPCA for the Lenna image analysis.
	
	On the other hand, we note that for both the Lenna image analysis and simulation studies, the single interpolated PRPCA demonstrates clear advantages not only compared to the classical RPCA, but also compared to the more adaptive PRPCA with LR interpolation. In other words, although the single interpolation matrix is not data-dependent, a simple equal-weights average smoothing mechanism would benefits for many image problems tremendously.

	After all, the Lenna image analysis further validates the advantage of PRPCA over RPCA for smooth image recovery. Especially when image is of high resolution and  complicated (large latent $\rho_s$), it is more beneficial to impose the smoothness structure and allow the neighborhood pixels to learn from each other.
	Such benefits could be significant not only for computation, but also for recovery accuracy.
	

	\begin{figure}[ht]
		\begin{center}
			\includegraphics[width=\columnwidth]{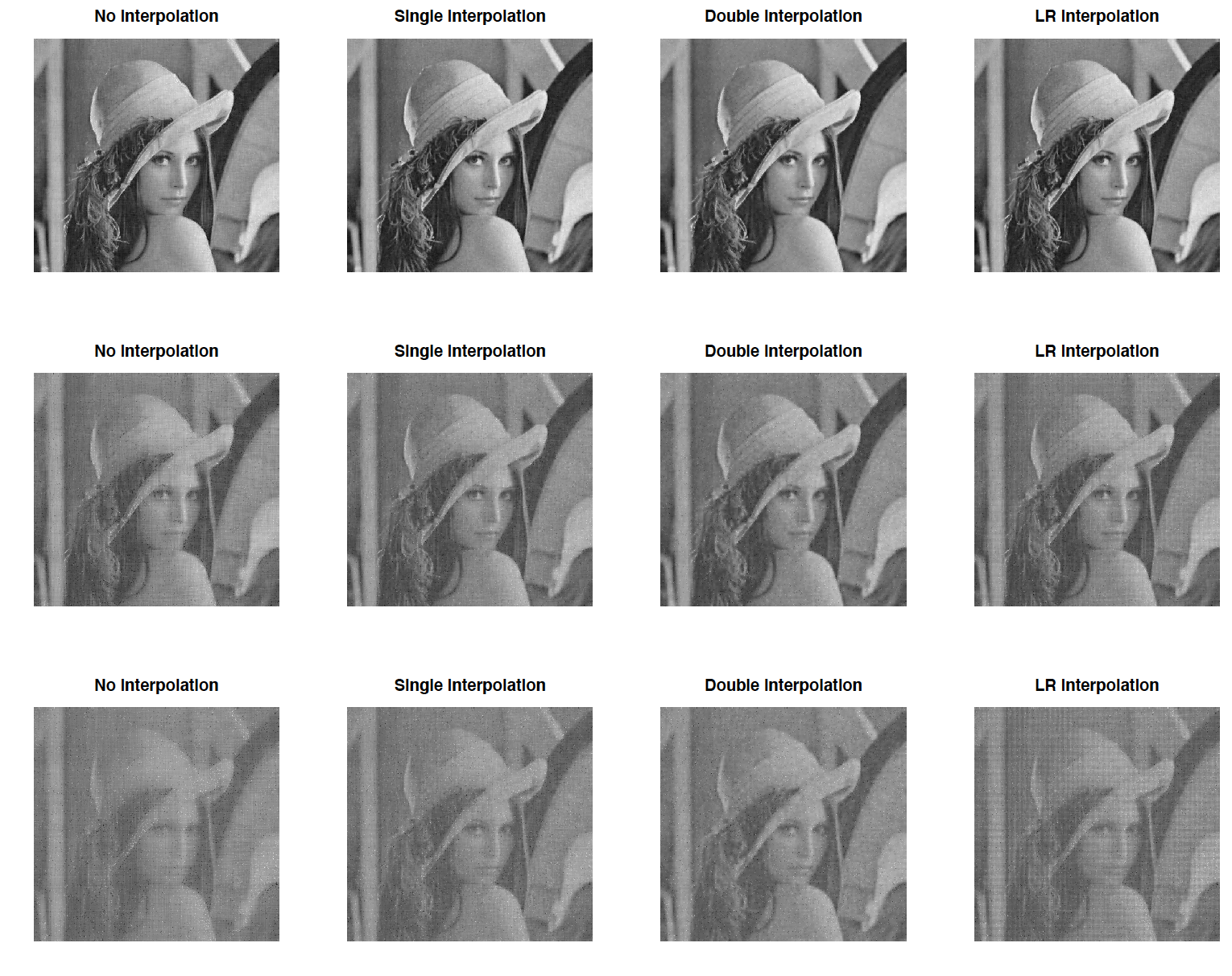}
			\vspace{-0.2in}
			\caption{Recovered Lenna image with no interpolation (first column), single interpolation (second column), double interpolation (third column), LR interpolation (last column) when, $\sigma=0.05$ (the first row), $\sigma=0.15$ (the middle row), $\sigma=0.25$ (the last row).}
			\label{lenna_blur}
		\end{center}
	\end{figure}
	
	\section{Conclusions and future work}\label{sec:conclusion}
	In this paper, we developed a novel framework of projected RPCA that motivated by smooth image recovery. This framework is general in the sense that it includes not only the classical RPCA as a special case, it also works for multivariate reduced rank regression with outliers. 
	Theoretically, we derived explicit error bounds on the estimation of $\bP\bX_0\bQ^\T$ and $\bY_0$. Our bounds match the optimum bounds in RPCA. In addition, by bringing the interpolation matrices into PRPCA model, 
	we could not only significantly speed up the computation of the RPCA,  but also improve matrix accuracy, which was demonstrated by a comprehensive simulation study and a real image data analysis. Due to the prevalence of low-rank and smooth images and (stacked) videos, this paper would greatly advance future research on many computer vision
	problems and demonstrate the potential of statistical methods on computer vision study.
	
	We conclude with the discussion of future works. One interesting direction is to explore the performance of PRPCA in a missing entry scenario. That is, when the entries of $\bZ$ are observed with both missingness and noise, how would the PRPCA perform in terms of matrix recovery accuracy compared to RPCA? 
	Consider an image inpainting problem, where images are observed with missing pixels. Intuitively, it would be more beneficial if 
	we could borrow information from the observed pixels for its neighbor missing entries. In other words, the interpolation matrices could play an even more significant role in image inpainting problems.
	Empirically, it is interesting to discover how different missing patterns and missing rates would affect the performance of PRPCA. Theoretically, it would also be significant to derive the error bounds under missing entry scenario. Such derivation may be more challenging as the dual certificate we constructed in Theorem \ref{thm-dual} may not be generalized directly.

	
	\bibliography{reference.bib}
	
	\appendix
	\section*{Appendix A.}
In the appendix, we provide proofs in the following order: Proposition \ref{prop-1}, Proposition \ref{prop-eta1}, Theorem \ref{thm-link}, Theorem \ref{thm-dual},  Theorem \ref{thm-main}. 
\medskip

\noindent {\bf Proof of Proposition  \ref{prop-1}.}
	We first show that the smallest singular value of interpolation matrices $\sigma_{\min}(\bP)$ is greater than 1. This is because for any $\bu\in \mathbb{R}^n$,
	\bes
	\|\bP\bu\|_2^2=	\sum_{j=1}^N(\bP_{j,\cdot}\bu)^2> \sum_{1\le j\le N \  \& \ j \text{ is even}}(\bP_{j,\cdot}\bu)^2 = \sum_{j=1}^n \bu_j^2 = \|\bu\|_2^2.
	\ees
	Similarly we have $\sigma_{\min}(\bQ)>1$. Then it follows that
	\bes
	\|\bP\bX\|_*\ge \sum_{i=1}^r \sigma_{i}(\bX)\sigma_{\min}(\bP)>\|\bX\|_*
	\ees
	where $r$ is the rank of $\bX$. Furthermore,
	\bes
	\|\bP\bX\bQ^\T\|_*\ge \sum_{i=1}^r \sigma_{i}(\bP\bX)\sigma_{\min}(\bQ)>\|\bP\bX\|_*>\|\bX\|_*.
	\ees
	This completes the proof.
\medskip
\\

\noindent {\bf Proof of Proposition  \ref{prop-eta1}.}
	When $\|\mP_{\mT_0^\perp}(\bX)\|_*=0$, we have $\bX=\mP_{\mT_0}(\bX)$, in other words, $\bX\in \mT_0$. Thus $\bX$ can be written as $\bX=\bU_0\bX_1^\T+\bX_2\bV_0^\T$ for certain matrices $\bX_1\in \mathbb{R}^{m\times r}$ and $\bX_2\in \mathbb{R}^{n\times r}$. It then follows that
	\begin{align}
	&\mP_{\mT}(\bP\bX\bQ^\T)
	\cr=& \mP_{\mT}(\bP\bU_0\bX_1^\T\bQ^\T+\bP\bX_2\bV_0^\T\bQ^\T)
	\cr=&\tbU_0\tbU_0^T\bP\bU_0\bX_1^\T\bQ^\T +\tbU_0\tbU_0^T\bP\bX_2\bV_0^\T\bQ^\T
	\cr& + \bP\bU_0\bX_1^\T\bQ^\T\tbV_0\tbV_0^\T + \bP\bX_2\bV_0^\T\bQ^\T\tbV_0\tbV_0^\T \cr &-\tbU_0\tbU_0^T\bP\bU_0\bX_1^\T\bQ^\T\tbV_0\tbV_0^\T +\tbU_0\tbU_0^T\bP\bX_2\bV_0^\T\bQ^\T\tbV_0\tbV_0^\T
	\cr=&\bP\bU_0\bX_1^\T\bQ^\T +\tbU_0\tbU_0^T\bP\bX_2\bV_0^\T\bQ^\T
	\cr&+ \bP\bU_0\bX_1^\T\bQ^\T\tbV_0\tbV_0^\T + \bP\bX_2\bV_0^\T\bQ^\T
	\cr &-\bP\bU_0\bX_1^\T\bQ^\T\tbV_0\tbV_0^\T +\tbU_0\tbU_0^T\bP\bX_2\bV_0^\T\bQ^\T
	\cr=&\bP\bU_0\bX_1^\T\bQ^\T+\bP\bX_2\bV_0^\T\bQ^\T
	\cr=&\bP\bX\bQ^\T
	\end{align}
	Therefore, $\mP_{\mT^\perp}(\bP\bX\bQ^\T)={\bf 0}$. This completes the proof.
\medskip
\\

\noindent{\bf Proof of Theorem  \ref{thm-link}.}
	The condition
	\bes
	\mP_{\mT}(\bD) = \bGamma
	\ees
	suggests that
	\begin{align}\label{pf-4-1}
	\tbU_0\tbU_0^T\bD + \bD\tbV_0\tbV_0^\T - \tbU_0\tbU_0^\T\bD\tbV_0\tbV_0^\T = \bGamma.
	\end{align}
	Left-multiply $\bU_0\bU_0^\T\bP^\T$ and right-multiply $\bQ$ on both sides of (\ref{pf-4-1}), we get 
	\begin{align}\label{pf-4-2}
	\bU_0\bU_0^\T\bP^\T\left(\tbU_0\tbU_0^T\bD + \bD\tbV_0\tbV_0^\T - \tbU_0\tbU_0^\T\bD\tbV_0\tbV_0^\T\right) \bQ= \bU_0\bU_0^\T\bP^\T\bGamma\bQ.
	\end{align}
	For the LHS of (\ref{pf-4-2}),
	\begin{align}\label{pf-4-3}
	&\bU_0\bU_0^\T\bP^\T\left(\tbU_0\tbU_0^T\bD + \bD\tbV_0\tbV_0^\T - \tbU_0\tbU_0^\T\bD\tbV_0\tbV_0^\T\right)\bQ
	\cr =&\left(\bU_0\bU_0^\T\bP^\T\bD+\bU_0\bU_0\bP^\T\bD\tbV_0\tbV_0^\T-\bU_0\bU_0\bP^\T\bD\tbV_0\tbV_0^\T\right)\bQ
	\cr =&\bU_0\bU_0^\T\bP^\T\bD\bQ,
	\end{align}
	For the RHS of (\ref{pf-4-2}),
	\begin{align}\label{pf-4-4}
	&\bU_0\bU_0^\T\bP^\T\bGamma\bQ
	\cr=&\bU_0\bU_0^\T\bP^\T\big((\bP\bU_0)^+\big)^\T\bV_0^\T\bQ^+\bQ+\bU_0\bU_0^\T\bP^\T(\bP^+)^\T\bU_0(\bQ\bV_0)^+\bQ\cr &-\bU_0\bU_0^\T\bP^\T\big((\bP\bU_0)^+\big)^\T(\bQ\bV_0)^+\bQ
	\cr=&\bU_0\bV_0^\T+\bU_0(\bQ\bV_0)^+\bQ-\bU_0(\bQ\bV_0)^+\bQ
	\cr=&\bU_0\bV_0^\T.
	\end{align}
	Plug (\ref{pf-4-3}) and (\ref{pf-4-4}) into (\ref{pf-4-2}), we have 
	\bel{pf-4-5}
	\bU_0\bU_0^\T\bP^\T\bD\bQ=\bU_0\bV_0^\T.
	\eel
	
	Similarly, left-multiply $\bP^\T$ and right-multiply $\bQ\bV_0\bV_0^\T$ on both sides of (\ref{pf-4-1}), we get 
	\bel{pf-4-6}
	\bP^\T\bD\bQ\bV_0\bV_0^\T=\bU_0\bV_0^\T.
	\eel
	
	Finally, left-multiply $\bU_0\bU_0^\T\bP^\T$ and right-multiply $\bQ\bV_0\bV_0^\T$ on both sides of (\ref{pf-4-1}), we get
	\bel{pf-4-7}
	\bU_0\bU_0^\T\bP^\T\bD\bQ\bV_0\bV_0^\T=\bU_0\bV_0^\T.
	\eel
	Combine (\ref{pf-4-5}), (\ref{pf-4-6}) and (\ref{pf-4-7}) together, we have
	\bes
	\bU_0\bU_0^\T\bP^\T\bD\bQ+\bP^\T\bD\bQ\bV_0\bV_0^\T-\bU_0\bU_0^\T\bP^\T\bD\bQ\bV_0\bV_0^\T=\bU_0\bV_0^\T.
	\ees
	In other words, 
	\bes
	\mP_{\mT_0}(\bP^\T\bD\bQ) = \bU_0\bV_0^\T.
	\ees
	This completes the proof.
\medskip
\\

Before proving Theorem \ref{thm-dual}, we first provide the following Definitions and Lemmas from \citet{hsu2011robust}.
\begin{definition}
	The matrix norm $\|\cdot\|_{\#'}$ is said to be the dual norm of $\|\cdot\|_\#$ if for all $\bM$,  $\|\bM\|_{\#'}=\sup_{\|N\|_\#\le 1}\langle \bM,\bN\rangle$.
\end{definition}
\begin{lemma}\label{lm-1}
	For any linear matrix operator $\mathcal{T}: \mathbb{R}^{n\times m} \rightarrow \mathbb{R}^{n\times m}$, and any pair of matrix norms $\|\cdot\|_{\#}$ and $\|\cdot\|_{*}$, we have
	\bes
	\| \mathcal{T}\|_{\#\rightarrow *} = 	\| \mathcal{T}^*\|_{*'\rightarrow \#'} 
	\ees
	where $\|\cdot\|_{\#'}$ is the dual norm of $\|\cdot\|_\#$ and $\|\cdot\|_{*'}$ is the dual norm of $\|\cdot\|_*$.
\end{lemma}
\begin{lemma}\label{lm-2}
	For any matrix $\bM \in \mathbb{R}^{n\times m}$ and $p\in\{1,\infty\}$, we have
	\bes
	\| \mathcal{P}_\mathcal{S}\|_{\tvec(\infty)\rightarrow \star(\rho)} \le \alpha(\rho).
	\ees
	where the norm $\|\cdot\|_{\star(\rho)}$ is defined as
	\bes
	\|\bM\|_{\star(\rho)}= \max\{\rho\|\bM\|_{1\rightarrow 1}, \rho^{-1}\|\bM\|_{\infty\rightarrow \infty}\}.
	\ees
\end{lemma}
\begin{lemma}\label{lm-3}
	For any matrix $\bM \in \mathbb{R}^{n\times m}$, we have for all $\rho>0$,
	\bes
	\| \bM\|_{2\rightarrow 2} \le \|\bM\|_{\star(\rho)}.
	\ees
\end{lemma}

\begin{lemma}\label{lm-4}
	For any matrix $\bM\in\mathbb{R}^{N\times M}$ and $p\in\{1,\infty\}$, we have	
	\begin{align*}
	&\|\mP_{\mS}(\bM)\|_{p\rightarrow p}\le 	\|\sgn(\bX_0)\|_{p\rightarrow p}	\|\bM\|_{\tvec(\infty)},
	\cr &\|\mP_{\mS}\|_{\tvec(\infty)\rightarrow \star(\rho)} \le \alpha(\rho),
	\cr&\|\mP_{\mT}(\bM)\|_{*}\le 2r\|\bM\|_{2\rightarrow 2},
	\cr&\|\mP_{\mT}(\bM)\|_{\tvec(2)}\le 2\sqrt{r}\|\bM\|_{2\rightarrow 2}.
	\end{align*}
\end{lemma}

\begin{lemma}\label{lm-5}
	For any linear matrix operator $\mathcal{T}_1 : \mathbb{R}^{n\times m} \rightarrow \mathbb{R}^{n\times m}$ and $\mathcal{T}_2 : \mathbb{R}^{n\times m} \rightarrow \mathbb{R}^{n\times m}$, and any matrix norm $\|\cdot\|_{\#}$, if $\|\mT_1\circ\mT_2\|_{\#} <1$, then $\mathcal{I}-\mT_1\circ\mT_2$ is convertible and satisfies
	\bes
	\| (\mathcal{I}-\mT_1\circ\mT_2)^{-1}\|_{\#\rightarrow \#} \le\frac{1}{	 1-\|\mT_1\circ\mT_2\|_{\#\rightarrow \#}},
	\ees
	where $\mathcal{I}$ is the identity operator.
\end{lemma}
\medskip
Now we are ready to prove Theorem  \ref{thm-dual}.
\medskip

\noindent{\bf Proof of Theorem  \ref{thm-dual}.}
	First, it is not hard to verify that $ \bD_\mS \in\mS $, $ \bD_\mT \in\mT $ and the first two equality of (\ref{dual-1}). The third equality of (\ref{dual-1}) followed by Theorem \ref{thm-link}. We now prove (\ref{dual-3}). 
	\bes
	&&\|\bD_\mS\|_{2\rightarrow 2}\cr &\le& \|\bD_\mS\|_{\star(\rho)}
	\cr  &=& \left\|(\mI-\mP_\mS\circ\mP_\mT)^{-1} \left(\lam_2\sgn(\bY_0)-\lam_1\mP_{\mS}(\bGamma)- (\mP_\mS\circ\mP_{\mT^\perp})(\bE)\right)\right\|_{\star(\rho)}
	\cr &\le& \frac{1}{1-\alpha(\rho)\beta(\rho)}\left\| \lam_2\sgn(\bY_0)-\lam_1\mP_{\mS}(\bGamma)- (\mP_\mS\circ\mP_{\mT^\perp})(\bE) \right\|_{\star(\rho)}
	\cr &=& \frac{1}{1-\alpha(\rho)\beta(\rho)}\Big(\lam_2\|\sgn(\bY_0)\|_{\star(\rho)} +\lam_1\left\|\mP_{\mS}(\bGamma)\right\|_{\star(\rho)} + \|(\mP_\mS\circ\mP_{\mT^\perp})(\bE)\|_{\star(\rho)}\Big)
	\cr &\le&\frac{\alpha(\rho)}{1-\alpha(\rho)\beta(\rho)}(\lam_2+\gamma_1\lam_1+ \eps_{\infty}),
	\ees
	where 
	the first inequality holds by Lemma \ref{lm-3}, the second inequality holds by Lemma \ref{lm-5}, the last inequality hods by Lemma \ref{lm-4} and
	\bes
	\|(\mP_\mS\circ\mP_{\mT^\perp})(\bE)\|_{\star(\rho)}\le\alpha(\rho) \|\mP_{\mT^\perp}(\bE)\|_{\infty}\le \alpha(\rho)(\|\bE\|_\infty+\|\mP_{\mT}(\bE)\|_{\infty})\le \alpha(\rho)\eps_{\infty}.
	\ees
	Similarly, for $\|\bD_\mT\|_{\infty}$,
	\bes
	&&\|\bD_\mT\|_{\infty}\cr &=&\left\| (\mI-\mP_\mT\circ\mP_\mS)^{-1} \Big(\lam_1\bGamma-\lam_2\mP_{\mT}\big(\sgn(\bY_0)\big)- (\mP_\mT\circ\mP_{\mS^\perp})(\bE)\Big)\right\|_{\infty}
	\cr &\le& \frac{1}{1-\alpha(\rho)\beta(\rho)} \left(\lam_1\|\bGamma\|_\infty+\lam_2\big\|\mP_{\mT}\big(\sgn(\bY_0)\big)\big\|_\infty+ \big\|(\mP_\mT\circ\mP_{\mS^\perp})(\bE)\big\|_\infty\right)
	\cr &\le &\frac{1}{1-\alpha(\rho)\beta(\rho)} \left(\gamma_1\lam_1+\lam_2 \alpha(\rho)\beta(\rho)+\eps_{\infty}\right),
	\ees
	where for the last inequality we used the bound
	\bes
	\big\|(\mP_\mT\circ\mP_{\mS^\perp})(\bE)\big\|_\infty\le \|\mP_\mT(\bE)-(\mP_\mT\circ\mP_{\mS})(\bE)\|_\infty\le \|\mP_\mT(\bE)\|+\alpha(\rho)\beta(\rho)\|(\bE)\|_\infty\le \eps_\infty.
	\ees
	For $\|\bD_\mT\|_{*}$, 
	\bes
	&&\|\bD_\mT\|_{*}
	\cr\le && r\|\bD_\mT\|_{2\rightarrow 2}
	\cr= &&r\|\mP_\mT(\bD_\mS+\bE) -\lam_1\bGamma\|_{2\rightarrow 2}
	\cr= &&r\left(\|\bD_\mS\|_{2\rightarrow 2}+\|\bE\|_{2\rightarrow 2} +\lam_1\gamma_2\right)
	\cr \le &&r\left(\frac{2\alpha(\rho)}{1-\alpha(\rho)\beta(\rho)}(\lam_2+\gamma_1\lam_1+ \eps_{\tvec(\infty)})+2\eps_{2\rightarrow 2}+\lam_1\gamma_2 \right).
	\ees
	For $\|\bD_\mS\|_{\tvec(1)}$, we have
	\bes
	&&\|\bD_\mS\|_{\tvec(1)}
	\cr\le && s\|\bD_\mS\|_{\tvec(\infty)}
	\cr \le &&\frac{s}{1-\alpha(\rho)\beta(\rho)}\Big(\lam_2\|\sgn(\bY_0)\|_{\tvec(\infty)} +\lam_1\big\|\mP_{\mS}(\bGamma)\big\|_{\tvec(\infty)} + \|(\mP_\mS\circ\mP_{\mT^\perp})(\bE)\|_{\tvec(\infty)}\Big)
	\cr \le && \frac{s}{1-\alpha(\rho)\beta(\rho)}(\lam_2+ \lam_1\gamma_1+\eps_{\infty}).
	\ees
	Finally,
	\bes
	\|\bD_\mT+\bD_\mS\|_2^2&=&\langle\bD_\mS,\mP_{\mS}(\bD_\mS+\bD_{\mT})\rangle + \langle\bD_\mT,\mP_{\mT}(\bD_\mS+\bD_{\mT})\rangle\cr &=& \langle\bD_\mS, \lam_2\mP_{\mS}(\sgn(\bY_0))-\mP_{\mS}(\bE)\rangle \cr &&+ \langle\bD_\mT, \lam_1 \mP_{\mT}(\bGamma)-\mP_{\mT}(\bE)\rangle
	\cr &\le & \|\bD_\mS\|_{\tvec(1)} (\lam_2+\eps_\infty)+\|\bD_\mT\|_{*} (\lam_1\gamma_2 + \eps_{2\rightarrow 2}).
	\ees
	This finish the proof for (\ref{dual-3}). To prove (\ref{dual-2}), let $\bD=\bD_\mS+\bD_\mT+\bE$,
	\bes
	\big\| \mP_{\mT_0^\perp}(\bP^\T\bD\bQ)\big\|_{2\rightarrow 2}&=&\Big\|\mP_{\mT_0^\perp}\left(\bP^\T(\bD_\mS+\bE)\bQ\right)\Big\|_{2\rightarrow 2}\cr &\le & \sigma_{\max}(\bP) \sigma_{\max}(\bQ)\left(\|\bD_\mS\|_{2\rightarrow 2}+\|\bE\|_{2\rightarrow 2 }\right)
	\cr &\le &\sigma_{\max}(\bP) \sigma_{\max}(\bQ)\left(\frac{\alpha(\rho)(\lam_2+\gamma_1\lam_1+ \eps_{\infty})}{1-\alpha(\rho)\beta(\rho)}+\eps_{2\rightarrow 2 }\right)\le \frac{\lam_1}{c},
	\ees
	where the last inequality holds by penalty condition (i). Similarly, we can bound $\big\| \mP_{\mS^\perp}(\bD)\big\|_{\tvec(\infty) }$ as below,
	\bes
	\big\| \mP_{\mS^\perp}(\bD)\big\|_{\tvec(\infty) }&=&\Big\|\mP_{\mS^\perp}\left(\bD_\mT+\bE\right)\Big\|_{\tvec(\infty)}
	\cr&\le& \|\bD_{\mT}\|_\infty +\eps_{\infty}
	\cr &\le& \frac{1}{1-\alpha(\rho)\beta(\rho)} \left(\gamma_1\lam_1+\lam_2 \alpha(\rho)\beta(\rho)+\eps_{\infty}\right) + \eps_{\infty} \le \frac{\lam_2}{c},
	\ees
	where the last inequality holds by penalty condition (ii). 
\medskip
\\

\noindent{\bf Proof of Theorem  \ref{thm-main}.}
	To prove Theorem \ref{thm-main}, we need the following Propositions.
	\begin{proposition}\label{prop-sbg}
		For any $\lam_1>0$, $\lam_2>0$, define the penalty function $\text{Pen} (\bX, \bY)=\lam_1\|\bX\|_*+\lam_2 \|\bY\|_{\tvec(1)}$ with domain $\bX\in\mathbb{R}^{n\times m}$ and $\bY\in\mathbb{R}^{N\times M}$.
		Then, if there exists $\bD$ satisfies 
		\bes
		\mP_{\mT}(\bD)&=& \lam_1\bGamma
		\cr \mP_{\mS}(\bD)&=& \lam_2 \  \sgn(\bY_0),
		\ees
		and $\mP_{\mT_0^\perp}(\bP^\T\bD\bQ)\|_{2\rightarrow 2} \le \lam_1/c$, $\|\mP_{\mS^\perp}(\bD)\|_{\tvec(\infty)} \le\lam_2/c$, we have
		\bes
		\lam_1	\left\|\bX_0+\bDelta_{\bX} \right\|_*-\lam_1\|\bX_0\|_*-\langle \bP^\T\bD\bQ,\bDelta_{\bX} \rangle\ge \lam_1(1-1/c)\|\mP_{\mT_0^\perp}(\bDelta_{\bX})\|_*,
		\ees
		and 
		\bes
		\lam_2\left\|\bY_0+\bDelta_{\bY} \right\|_*-\lam_2\left\|\bY_0\right\|_*-\langle \bD,\bDelta_{\bY} \rangle\ge \lam_2(1-1/c)\big\|\mP_{\mS^\perp}(\bDelta_{\bY})\big\|_{\tvec(1)}.
		\ees
	\end{proposition}
	\medskip
	
\noindent{\bf Proof of Proposition  \ref{prop-sbg}.}
		First, by the construction of $\bD$ and Theorem \ref{thm-link}, we have
		$\mP_{\mT_0}(\bP^T\bD\bQ)=\bU_0\bV_0^\T$. 
		On the other hand, for any other sub-gradient $\bG\in\partial_{\bX}(\lam_1\|\bX_0\|_*)$, we have 
		\bes
		\mP_{\mT_0}(\bG)=\bU_0\bV_0^\T.
		\ees
		It then follows that
		\bes
		\bG-\bP^\T\bD\bQ&=&\mP_{\mT_0}(\bG) + \mP_{\mT_0^\perp}(\bG)- \mP_{\mT_0}(\bP^\T\bD\bQ)- \mP_{\mT_0^\perp}(\bP^\T\bD\bQ)
		\cr &=&\mP_{\mT_0^\perp}(\bG)-  \mP_{\mT_0^\perp}(\bP^\T\bD\bQ).
		\ees
		As a consequence,
		\begin{align}\label{sbg-1}
		&\lam_1\left\|\bX_0+\bDelta_{\bX} \right\|_*-\lam_1 \|\bX_0\|_*-\langle \bP^\T\bD\bQ,\bDelta_{\bX} \rangle
		\cr \ge &\sup\left\{\langle \bG,\bDelta_{\bX}\rangle -\langle \bP^\T\bD\bQ,\bDelta_{\bX}\rangle :\ \ \bG\in \partial_{\bX} (\lam_1\|\bX_0\|_*)\right\}
		\cr=&\sup\left\{\langle \bG-\bP^\T\bD\bQ,\bDelta_{\bX}\rangle :\ \ \bG\in \partial_{\bX} (\lam_1\|\bX_0\|_*)\right\}
		\cr=&\sup\left\{\langle \mP_{\mT_0^\perp}(\bG-\bP^\T\bD\bQ),\bDelta_{\bX}\rangle :\ \ \bG\in \partial_{\bX} (\lam_1\|\bX_0\|_*)\right\}
		\cr=&\sup\left\{\langle \mP_{\mT_0^\perp}(\bG),\mP_{\mT_0^\perp}(\bDelta_{\bX})\rangle -\langle \mP_{\mT_0^\perp}(\bP^\T\bD\bQ),\bDelta_{\bX}\rangle:\ \ \bG\in \partial_{\bX} (\lam_1\|\bX_0\|_*)\right\}
		\cr\ge&\lam_1\|\mP_{\mT_0^\perp}(\bDelta_{\bX})\|_*-\| \mP_{\mT_0^\perp}(\bP^\T\bD\bQ)\|_{2\rightarrow 2} \|\mP_{\mT_0^\perp}(\bDelta_{\bX})\|_*
		\cr\ge&\lam_1(1-1/c)\|\mP_{\mT_0^\perp}(\bDelta_{\bX})\|_*.
		\cr
		\end{align}
		Similarly,
		\begin{align}\label{sbg-2}
		&\lam_2\left\|\bY_0+\bDelta_{\bY} \right\|_*-\lam_2\left\|\bY_0\right\|_*-\langle \bD,\bDelta_{\bY} \rangle
		\cr \ge &\sup\left\{\langle \bG,\bDelta_{\bY}\rangle -\langle \bD,\bDelta_{\bY}\rangle :\ \ \bG\in \partial_{\bY} (\lam_2\|\bY_0\|_{\tvec(1)})\right\}
		\cr = &\sup\left\{\big\langle \left\{\mP_{\mS^\perp}(\bG)-\mP_{\mS^\perp}( \bD)\right\},\ \bDelta_{\bY}\big\rangle :\ \ \bG\in \partial_{\bY} (\lam_1\|\bY_0\|_{\tvec(1)})\right\}
		\cr \ge &\lam_2\big\|\mP_{\mS^\perp}(\bDelta_{\bY})\big\|_{\tvec(1)}-\big\langle \mP_{\mS^\perp}(\bD),\ \mP_{\mS^\perp}(\bDelta_{\bY})\big\rangle 
		\cr \ge &\lam_2\big\|\mP_{\mS^\perp}(\bDelta_{\bY})\big\|_{\tvec(1)}-\big\| \mP_{\mS^\perp}(\bD)\big\|_{\tvec(\infty)} \big\|\mP_{\mS^\perp}(\bDelta_{\bX})\big\|_{\tvec(1)}
		\cr \ge &\lam_2(1-1/c)\big\|\mP_{\mS^\perp}(\bDelta_{\bY})\big\|_{\tvec(1)}.
		\end{align}
				This completes the proof of Proposition \ref{prop-sbg}.
	\medskip
	
	\begin{proposition}\label{prop-bound}
		Suppose the conditions in Theorem \ref{thm-dual} holds and let $\bD_{\mS}$ and $\bD_{\mT}$ be as in Theorem \ref{thm-dual}, then
		\bes
		\lam_1\|\mP_{\mT_0^\perp}(\bDelta_{\bX})\|_*+\lam_2\big\|\mP_{\mS^\perp}(\bDelta_{\bY})\big\|_{\tvec(1)}\le(1-1/c)^{-1}(1/2)\|\bD_{\mT}+\bD_{\mS}\|_2^2
		\ees
	\end{proposition}
	
	\noindent{\bf Proof of Proposition \ref{prop-bound}.}
		By the optimality of $(\hbX, \hbY)$, we have
		\bel{pf-3}
		&&\lambda_1(\|\hbX\|_*-\|\bX_0\|_*)  +\lambda_2(\|\hbY\|_1-\|\bY_0\|_1) \cr \le &&-\frac{1}{2}\big\|\widehat{\bTheta}-\bTheta\big\|^2 +\langle\bE, \bP\bDelta_X\bQ^\T \rangle +\langle \bE, \bDelta_Y\rangle
		\eel
		Combining (\ref{sbg-1}), (\ref{sbg-2}) and (\ref{pf-3}), we have
		\bes
		&&\lam_1(1-1/c)\|\mP_{\mT_0^\perp}(\bDelta_{\bX})\|_*+\lam_2(1-1/c)\big\|\mP_{\mS^\perp}(\bDelta_{\bY})\big\|_{\tvec(1)}
		\cr \le &&- \langle \bP^\T(\bD_\mT+\bD_\mS) \bQ,\bDelta_{\bX} \rangle - \langle \bD_\mT+\bD_\mS,\bDelta_{\bY} \rangle - \frac{1}{2}\|\widehat{\Theta}-\Theta\|_F^2
		\cr = &&- \langle \bD_\mT+\bD_\mS,\bP\bDelta_{\bX}\bQ^\T \rangle - \langle \bD_\mT+\bD_\mS,\bDelta_{\bY} \rangle - \frac{1}{2}\|\bP\bDelta_{\bX}\bQ^\T+\bDelta_{\bY}\|_F^2
		\cr \le && \frac{1}{2}\|\bD_{\mT}+\bD_{\mS}\|_2^2. 
		\ees
		This completes the proof of Proposition \ref{prop-bound}.
\medskip

	Now we are ready to prove Theorem \ref{thm-main}. By KKT condition, 
	\bel{pf-main-1}
	\lam_1 \bG_X&=&\bP^\T(\bZ-\bP\hbX\bQ^\T-\hbY)\bQ\cr
	\lam_2\bG_Y&=&\bZ-\bP\hbX\bQ^\T-\hbY
	\eel
	By algebra,
	\bel{pf-main-2}
	\lam_2  \bG_Y&=&\bE -\bP\bDelta_{\bX}\bQ^\T -\bDelta_{\bY}
	\eel
	Left-multiply $\bP^*$ and right-multiply $\bQ^*$ on both sides, 
	\bel{pf-main-3}
	\lam_2  \bP^*\bG_Y\bQ^*&=&\bP^*\bE\bQ^* -\bP\bDelta_{\bX}\bQ^\T -\bP^*\bDelta_{\bY}\bQ^*
	\eel
	Project both sides of (\ref{pf-main-3}) onto $\mS$, we have 
	\begin{equation}\label{pf-main-4}
	\lam_2 \mP_{\mS}(\bP^*\bG_Y\bQ^*)=\mP_{\mS}(\bP^*\bE\bQ^*) -\mP_{\mS}(\bP\bDelta_{\bX}\bQ^\T) -\mP_{\mS}(\bP^*\bDelta_{\bY}\bQ^*),
	\end{equation}
	Similarly,
	\bel{pf-main-5}
	\lam_1  \bG_X&=&\bP^\T\bE\bQ -\bP^\T\bP\bDelta_{\bX}\bQ^\T\bQ  -\bP^\T\bDelta_{\bY}\bQ.
	\eel
	Multiply above by $(\bP^+)^\T$ and $\bQ^+$ on left and right respectively, 
	\begin{equation}\label{pf-main-6}
	\lam_1  (\bP^+)^\T\bG_X\bQ^+=\bP^*\bE\bQ^* -\bP\bDelta_{\bX}\bQ^\T  -\bP^*\bDelta_{\bY}\bQ^*.
	\end{equation}
	Project onto $\mT$ on both sides, we obtain
	\begin{align}\label{pf-main-7}
	&\lam_1  \mP_{\mT} \left( (\bP^+)^\T\bG_X\bQ^+\right)\cr =& \mP_{\mT}(\bP^*\bE\bQ^*) -  \mP_{\mT} (\bP\bDelta_{\bX}\bQ^\T)  -\mP_{\mT}(\bP^*\bDelta_{\bY}\bQ^*)
	\end{align}
	Further project onto $\mS$ on both sides of (\ref{pf-main-7}),
	\begin{align}\label{pf-main-8}
	&\lam_1   (\mP_{\mS}\circ\mP_{\mT}) \left( (\bP^+)^\T\bG_X\bQ^+\right)\cr =& (\mP_{\mS}\circ\mP_{\mT})(\bP^*\bE\bQ^*) -  (\mP_{\mS}\circ\mP_{\mT}) (\bP\bDelta_{\bX}\bQ^\T)  \cr &-(\mP_{\mS}\circ\mP_{\mT})(\bP^*\bDelta_{\bY}\bQ^*)
	\end{align}
	Subtracting (\ref{pf-main-8}) from (\ref{pf-main-4}) we have
	\begin{align}\label{pf-main-9}
	&\mP_{\mS}(\bP^*\bDelta_{\bY}\bQ^*)- (\mP_{\mS}\circ\mP_{\mT})(\bP^*\bDelta_{\bY}\bQ^*)+ (\mP_{\mS}\circ\mP_{\mT^\perp})(\bP\bDelta_{\bX}\bQ^\T)\cr =&  (\mP_{\mS}\circ\mP_{\mT^\perp})(\bP^*\bE\bQ^*) - \lam_2  \mP_{\mS}(\bP^*\bG_Y\bQ^*)+  \lam_1   (\mP_{\mS}\circ\mP_{\mT}) \left( (\bP^+)^\T\bG_X\bQ^+\right).
	\cr
	\end{align}
	As $\langle\sgn\left((\bP^*\bDelta_Y\bQ^*)_{\mS}\right), \mP_{\mS}(\bP^*\bDelta_Y\bQ^*)\rangle=\|\mP_{\mS}(\bP^*\bDelta_Y\bQ^*)\|_{\tvec(1)}$, take inner product on both sides of (\ref{pf-main-9}), we have
	\begin{align*}
	&\|\mP_{\mS}(\bP^*\bDelta_{\bY}\bQ^*)\|_{\tvec(1)}
	\cr \le & \|(\mP_{\mS}\circ\mP_{\mT})(\bP^*\bDelta_{\bY}\bQ^*)\|_{\tvec(1)}+ \|(\mP_{\mS}\circ\mP_{\mT^\perp})(\bP\bDelta_{\bX}\bQ^\T)\|_{\tvec(1)} 
	\cr &+\|(\mP_{\mS}\circ\mP_{\mT^\perp})(\bP^*\bE\bQ^*)\|_{\tvec(1)} + \lam_2  \|\mP_{\mS}(\bP^*\bG_Y\bQ^*)\|_{\tvec(1)}\cr &+  \lam_1   \|(\mP_{\mS}\circ\mP_{\mT}) \left( (\bP^+)^\T\bG_X\bQ^+\right)\|_{\tvec(1)}
	\cr \le& \alpha(\rho)\beta(\rho)\|\bP^*\bDelta_{\bY}\bQ^*\|_{\tvec(1)} 
	\cr &+ \sqrt{s} \|\mP_{\mT^\perp}(\bP\bDelta_{\bX}\bQ^\T)\|_{\tvec(2)}+\|(\mP_{\mS}\circ\mP_{\mT^\perp})(\bP^*\bE\bQ^*)\|_{\tvec(1)} 
	\cr & + \lam_2 s + \lam_1  \sqrt{s}\|\mP_{\mT}\left( (\bP^+)^\T\bG_X\bQ^+\right)\|_{\tvec(2)}
	\cr \le& \alpha(\rho)\beta(\rho)\|\mP_{\mS}(\bP^*\bDelta_{\bY}\bQ^*)\|_{\tvec(1)} + \alpha(\rho)\beta(\rho)\|\mP_{\mS^\perp}(\bP^*\bDelta_{\bY}\bQ^*)\|_{\tvec(1)}
	\cr &+ \sqrt{s} \|\mP_{\mT^\perp}(\bP\bDelta_{\bX}\bQ^\T)\|_{*}+s\|\mP_{\mT^\perp}(\bP^*\bE\bQ^*)\|_{\tvec(\infty)} 
	\cr & + \lam_2 s + 2\sigma_{\min}^{-1}(\bP)\sigma^{-1}_{\min}(\bQ)\lam_1 \sqrt{sr},
	\end{align*}
	where in the last inequality, we used the Lemma \ref{lm-4} and the inequalities $\|\bP^+\|_{2\rightarrow 2}\le \sigma_{\min}^{-1}(\bP)$,  $\|\bQ^+\|_{2\rightarrow 2}\le \sigma_{\min}^{-1}(\bQ)$.
	Rearrange above inequality, we obtain
	\begin{align}\label{pf-main-10}
	&\left(1-\alpha(\rho)\beta(\rho)\right)\|\mP_{\mS}(\bP^*\bDelta_{\bY}\bQ^*)\|_{\tvec(1)}
	\cr \le & \alpha(\rho)\beta(\rho)\|\mP_{\mS^\perp}(\bP^*\bDelta_{\bY}\bQ^*)\|_{\tvec(1)}  +\sqrt{s}\|\mP_{\mT^\perp}(\bP\bDelta_{\bX}\bQ^\T)\|_{*}
	\cr &+s\|\mP_{\mT^\perp}(\bP^*\bE\bQ^*)\|_{\tvec(\infty)} 
	+ \lam_2 s +2\sigma_{\min}^{-1}(\bP)\sigma^{-1}_{\min}(\bQ)\lam_1 \sqrt{sr}.
	\end{align}
	To bound 
	$\alpha(\rho)\beta(\rho)\|\mP_{\mS^\perp}(\bP^*\bDelta_{\bY}\bQ^*)\|_{\tvec(1)}  +\sqrt{s}\|\mP_{\mT^\perp}(\bP\bDelta_{\bX}\bQ^\T)\|_{*}$, we subtract (\ref{pf-main-2}) from (\ref{pf-main-3}), we have
	\begin{equation}\label{pf-main-11}
	\bDelta_{\bY}=\lam_2  \bP^*\bG_Y\bQ^*-\lam_2\bG_Y + \bE-\bP^*\bE \bQ^*+\bP^*\bDelta_{\bY}\bQ^*
	\end{equation}
	Project both sides of (\ref{pf-main-11}) onto $\mS$, we have 
	\begin{align}\label{pf-main-12}
	\mP_{\mS}(\bDelta_{\bY})=&\lam_2  \mP_{\mS}\left(\bP^*\bG_Y\bQ^*\right)-\lam_2\mP_{\mS}\left(\bG_Y\right)
	\cr&+ \mP_{\mS}(\bE)-\mP_{\mS}(\bP^*\bE\bQ^*) +\mP_{\mS}(\bP^*\bDelta_{\bY}\bQ^*).
	\end{align}
	It then follows that 
	\begin{align}\label{pf-main-13}
	&\|\mP_{\mS}(\bDelta_{\bY})\|_{\tvec(1)}
	\cr \le&\lam_2  \left\|\mP_{\mS}(\bP^*\bG_Y\bQ^*)\right\|_{\tvec(1)} + \lam_2  \left\|\mP_{\mS}(\bG_Y)\right\|_{\tvec(1)}+ \left\|\mP_{\mS}(\bE)\right\|_{\tvec(1)}
	\cr&+\left\|\mP_{\mS}(\bP^*\bE\bQ^*)\right\|_{\tvec(1)} +\left\|\mP_{\mS}(\bP^*\bDelta_{\bY}\bQ^*)\right\|_{\tvec(1)}
	\cr \le &2\lam_2 s +s\left\|\bE\right\|_{\tvec(\infty)}+s\left\|\bP^*\bE\bQ^*\right\|_{\tvec(\infty)}+\left\|\mP_{\mS}(\bP^*\bDelta_{\bY}\bQ^*)\right\|_{\tvec(1)}.
	\end{align}
	On the other hand, by the definition of  $\eta_2$, we have
	\begin{equation}\label{pf-main-14}
	\|\bP^*\bDelta_{\bY}\bQ^*\|_{\tvec(1)} \le  \|\mP_{\mS}(\bDelta_{\bY})\|_{\tvec(1)}+\eta_2^{-1}\|\mP_{\mS^\perp}(\bDelta_{\bY})\|_{\tvec(1)}.
	\end{equation}
	Combine (\ref{pf-main-13}) and (\ref{pf-main-14}), it follows that
	\bel{pf-main-15}
	\left\|\mP_{\mS^\perp}(\bP^*\bDelta_{\bY}\bQ^*)\right\|_{\tvec(1)} &\le&  \eta_2^{-1}\|\mP_{\mS^\perp}(\bDelta_{\bY})\|_{\tvec(1)}  + 2\lam_2 s
	\cr &&+s\left\|\bE\right\|_{\tvec(\infty)} +s\left\|\bP^*\bE\bQ^*\right\|_{\tvec(\infty)}.
	\eel
	Similarly, by  the definition of $\eta_1$, 
	\bel{pf-main-16}
	\|\mP_{\mT^\perp}(\bP\bDelta_{\bX}\bQ^\T)\|_{*}\le \eta_1^{-1} \|\mP_{\mT_0^\perp}(\bDelta_{\bX})\|_{*}.
	\eel
	Combine (\ref{pf-main-15}) and (\ref{pf-main-16}) and Proposition \ref{prop-bound}, we have
	\begin{align}\label{pf-main-17}
	&\alpha(\rho)\beta(\rho)\|\mP_{\mT^\perp}(\bP\bDelta_{\bX}\bQ^\T)\|_{*}+\sqrt{s}\left\|\mP_{\mS^\perp}(\bP^*\bDelta_{\bY}\bQ^*)\right\|_{\tvec(1)} 
	\cr \le & \left(\frac{\alpha(\rho)\beta(\rho)}{2\lam_2\eta_2} \vee  \frac{\sqrt{s}}{2\lam_1\eta_1}\right) (1-1/c)^{-1} \|\bD_{\mT}+\bD_{\mS}\|_2^2
	\cr &  + 2\lam_2 s+s\left\|\bE\right\|_{\tvec(\infty)} +s\left\|\bP^*\bE\bQ^*\right\|_{\tvec(\infty)}
	\cr\le &[2(1-1/c)\eta_0\lam_2]^{-1} \|\bD_{\mT}+\bD_{\mS}\|_2^2
	\cr &  + 2\lam_2 s+s\left\|\bE\right\|_{\tvec(\infty)} +s\left\|\bP^*\bE\bQ^*\right\|_{\tvec(\infty)}.
	\end{align}
	where we used the fact $\alpha(\rho)\beta(\rho)<1$, $\alpha(\rho)\ge \sqrt{s}$ and $\lam_1\ge \lam_2\alpha(\rho)\sigma_{\max}(\bP) \sigma_{\max}(\bQ)$ by (\ref{lambda2}).
	Further combine (\ref{pf-main-10}) and (\ref{pf-main-17}), we have
	\begin{align}\label{pf-main-18}
	&\left(1-\alpha(\rho)\beta(\rho)\right)\|\mP_{\mS}(\bP^*\bDelta_{\bY}\bQ^*)\|_{\tvec(1)}
	\cr\le &[2(1-1/c)\eta_0\lam_2]^{-1} \|\bD_{\mT}+\bD_{\mS}\|_2^2+ 3\lam_2 s+s\left\|\bE\right\|_{\tvec(\infty)}
	\cr & +s\left\|\bP^*\bE\bQ^*\right\|_{\tvec(\infty)}+s\|\mP_{\mT^\perp}(\bP^*\bE\bQ^*)\|_{\tvec(\infty)} 
	\cr & +2\sigma_{\min}^{-1}(\bP)\sigma^{-1}_{\min}(\bQ)\lam_1 \sqrt{sr}.
	\end{align}
	It further follows from (\ref{pf-main-18}), (\ref{pf-main-15}) and and Proposition \ref{prop-bound} that
	\bel{pf-main-19}
	&&\left(1-\alpha(\rho)\beta(\rho)\right)\|\bP^*\bDelta_{\bY}\bQ^*\|_{\tvec(1)}
	\cr\le&&
	\left(1-\alpha(\rho)\beta(\rho)\right)\left(\|\mP_{\mS}(\bP^*\bDelta_{\bY}\bQ^*)\|_{\tvec(1)}+ \|\mP_{\mS^\perp}(\bP^*\bDelta_{\bY}\bQ^*)\|_{\tvec(1)}\right)
	\cr\le &&[2\lam_2(1-1/c)]^{-1}(\eta_0^{-1}+\eta_2^{-1}) \|\bD_{\mT}+\bD_{\mS}\|_2^2
	\cr && + 5\lam_2 s+2s\left\|\bE\right\|_{\tvec(\infty)} +2s\left\|\bP^*\bE\bQ^*\right\|_{\tvec(\infty)}
	\cr &&+s\|\mP_{\mT^\perp}(\bP^*\bE\bQ^*)\|_{\tvec(\infty)} 
	+2\sigma_{\min}^{-1}(\bP)\sigma^{-1}_{\min}(\bQ)\lam_1 \sqrt{sr}.
	\cr \le &&[\lam_2(1-1/c)\eta_0]^{-1} \|\bD_{\mT}+\bD_{\mS}\|_2^2
	\cr &&+ 5\lam_2 s +2s\eps_{\infty}+3s\eps'_{\infty}  +2\sigma_{\min}^{-1}(\bP)\sigma^{-1}_{\min}(\bQ)\lam_1 \sqrt{sr}.
	\eel
	This proves (\ref{thm-main-1}). To prove  (\ref{thm-main-2}), we note that by (\ref{pf-main-13}), (\ref{pf-main-18}) and Proposition \ref{prop-bound},
	\begin{align}\label{pf-main-21}
	&(1-\alpha(\rho)\beta(\rho))\|\bDelta_{\bY}\|_{\tvec(1)}
	\cr \le &(1-\alpha(\rho)\beta(\rho))\|\mP_{\mS}(\bDelta_{\bY})\|_{\tvec(1)} + \|\mP_{\mS^\perp}(\bDelta_{\bY})\|_{\tvec(1)}
	\cr \le &2\lam_2 s +s\left\|\bE\right\|_{\tvec(\infty)}+s\left\|\bP^*\bE\bQ^*\right\|_{\tvec(\infty)}\cr &+(1-\alpha(\rho)\beta(\rho))\left\|\mP_{\mS}(\bP^*\bDelta_{\bY}\bQ^*)\right\|_{\tvec(1)} + [2(1-1/c)\lam_2]^{-1}\|\bD_{\mT}+\bD_{\mS}\|_2^2
	\cr \le &5\lam_2 s +2s\left\|\bE\right\|_{\tvec(\infty)}+2s\left\|\bP^*\bE\bQ^*\right\|_{\tvec(\infty)}+s\|\mP_{\mT^\perp}(\bP^*\bE\bQ^*)\|_{\tvec(\infty)} \cr &+2\sigma_{\min}^{-1}(\bP)\sigma^{-1}_{\min}(\bQ)\lam_1 \sqrt{sr}+ [2\lam_2(1-1/c)]^{-1}(\eta_0^{-1}+1)\|\bD_{\mT}+\bD_{\mS}\|_2^2
	\cr \le &5\lam_2 s +2s\left\|\bE\right\|_{\tvec(\infty)}+3s\eps'_{\infty}+2\sigma_{\min}^{-1}(\bP)\sigma^{-1}_{\min}(\bQ)\lam_1 \sqrt{sr}\cr &+ [2\lam_2(1-1/c)]^{-1}(\eta_0^{-1}+1)\|\bD_{\mT}+\bD_{\mS}\|_2^2.
	\cr
	\end{align}
	Finally, to prove (\ref{thm-main-3}), note that by (\ref{pf-main-7}),
	\begin{align}\label{pf-main-20}
	&\|\mP_{\mT} (\bP\bDelta_{\bX}\bQ^\T)\|_*\cr \le& \lam_1  \left\|\mP_{\mT} \left( (\bP^+)^\T\bG_X\bQ^+\right)\right\|_* + \|\mP_{\mT}(\bP^*\bE\bQ^*)\|_*   +\|\mP_{\mT}(\bP^*\bDelta_{\bY}\bQ^*)\|_*
	\cr\le &2\sigma_{\min}^{-1}(\bP)\sigma^{-1}_{\min}(\bQ)\lam_1 r+ \eps_* + \sqrt{2r}\|\bP^*\bDelta_{\bY}\bQ^*\|_{\tvec(2)}.
	\end{align}
	Then (\ref{thm-main-3}) follows from above. 

\end{document}

%% file: 04-pre-am-short.tex
\newcommand{\bel}{\begin{eqnarray}\label}
\newcommand{\eel}{\end{eqnarray}}
\newcommand{\bes}{\begin{eqnarray*}}
\newcommand{\ees}{\end{eqnarray*}}
\newcommand{\bei}{\begin{itemize}}
\newcommand{\beiftnt}{\begin{itemize}\footnotesize}
\newcommand{\eei}{\end{itemize}}

\def\benu{\begin{enumerate}}
\def\eenu{\end{enumerate}}

\def\argmin{\mathop{\rm arg\, min}}

\def\complex{\mathop{{\rm I}\kern-.58em\hbox{\rm C}}\nolimits}

\def\diag{\hbox{\rm diag}}

\def\rank{\hbox{\rm rank}}

\def\sgn{\hbox{\rm sgn}}

\def\supp{\hbox{\rm supp}}

\def\mathbold{\boldsymbol} 

\def\bA{\mathbold{A}}

\def\bB{\mathbold{B}}

\def\bD{\mathbold{D}}

\def\bE{\mathbold{E}}

\def\bF{\mathbold{F}}

\def\bG{\mathbold{G}}

\def\bI{\mathbold{I}}

\def\bJ{\mathbold{J}}

\def\bK{\mathbold{K}}

\def\bM{\mathbold{M}}

\def\bN{\mathbold{N}}

\def\bP{\mathbold{P}}\def\hbP{{\widehat{\bP}}}

\def\bQ{\mathbold{Q}}\def\hbQ{{\widehat{\bQ}}}

\def\bu{\mathbold{u}}

\def\bU{\mathbold{U}}\def\tbU{{\widetilde{\bU}}}

\def\bv{\mathbold{v}}

\def\bV{\mathbold{V}}\def\tbV{{\widetilde{\bV}}}

\def\bW{\mathbold{W}}

\def\bX{\mathbold{X}}\def\hbX{{\widehat{\bX}}}

\def\bY{\mathbold{Y}}\def\hbY{{\widehat{\bY}}}

\def\bZ{\mathbold{Z}}

\def\bGamma{\mathbold{\Gamma}}

\def\bDelta{\mathbold{\Delta}}

\def\eps{\epsilon}

\def\bTheta{\mathbold{\Theta}}\def\hbTheta{{\widehat{\bTheta}}}

\def\lam{\lambda}

\def\bLambda{\mathbold{\Lambda}}

\def\bSigma{\mathbold{\Sigma}}

\def\bOmega{\mathbold{\Omega}}